\documentclass{article}

\usepackage[margin=1in]{geometry}

\usepackage{tikz}
\usepackage{array}
\usepackage{multirow}
\usepackage{tabularx}
\usepackage{booktabs}

\usepackage[utf8]{inputenc} 
\usepackage[T1]{fontenc}    
\usepackage{hyperref}       
\usepackage{url}            
\usepackage{booktabs}       
\usepackage{amsfonts}       
\usepackage{nicefrac}       
\usepackage{microtype}      
\usepackage{amsmath,amssymb}
\usepackage{amsthm}
\usepackage{dsfont}

\usepackage[ruled,vlined]{algorithm2e}
\newcommand{\indicator}{\mathds{1}}
\newcommand{\ep}{\hfill $\Box$}

\newcommand{\skl}{\textnormal{kl}}
\newcommand{\E}{\mathbb{E}}

\newcommand{\tends}{\underset{t \to \infty}{\longrightarrow}}
\DeclareMathOperator*{\argmax}{arg\,max}
\DeclareMathOperator*{\argmin}{arg\,min}
\newtheorem{theorem}{Theorem}
\newtheorem{definition}{Definition}
\newtheorem{lemma}{Lemma}
\newtheorem{remark}{Remark}
\newtheorem{proposition}{Proposition}
\newtheorem{corollary}{Corollary}

\theoremstyle{definition}

\title{Optimal Best-arm Identification in Linear Bandits}

\date{}

\author{Yassir Jedra\thanks{Y. Jedra and A. Proutiere are with the Division of Decision and Control Systems, School of Electrical Engineering and Computer Science, Royal institute of Technology (KTH), Stockholm, Sweden. Emails: \{{\it jedra@kth.se, alepro@kth.se}\}.}  \and Alexandre Proutiere\footnotemark[1]
}

\begin{document}

\maketitle

\begin{abstract}

We study the problem of best-arm identification with fixed confidence in stochastic linear bandits. The objective is to identify the best arm with a given level of certainty while minimizing the sampling budget. We devise a simple algorithm whose sampling complexity matches known instance-specific lower bounds, asymptotically almost surely and in expectation. The algorithm relies on an arm sampling rule that tracks an optimal proportion of arm draws, and that remarkably can be updated as rarely as we wish, without compromising its theoretical guarantees. Moreover, unlike existing best-arm identification strategies, our algorithm uses a stopping rule that does not depend on the number of arms. Experimental results suggest that our algorithm significantly outperforms existing algorithms. The paper further provides a first analysis of the best-arm identification problem in linear bandits with a continuous set of arms.

\end{abstract}


\section{Introduction}

The stochastic linear bandit \cite{auer2003,dani2008stochastic} is a sequential decision-making problem that generalizes the classical stochastic Multi-Armed Bandit (MAB) problem \cite{robbins1952,lai1985} by assuming that the average reward is a linear function of the arm. Linear bandits have been extensively applied in online services such us online advertisement and recommendation systems \cite{li2010,chu2011,li2016collaborative}, and constitute arguably the most relevant {\it structured} bandit model in practice. Most existing analyses of stochastic linear bandits concern regret minimization \cite{dani2008stochastic, paat2010, abbasi2011improved, lattimore2016end, combes2017}, i.e., the problem of devising an online algorithm maximizing the expected reward accumulated over a given time horizon. When the set of arms is finite, this problem is {\it solved} in the sense that we know an instance-specific regret lower bound, and a simple algorithm whose regret matches this fundamental limit \cite{lattimore2016end, combes2017}. 

The best-arm identification problem (also referred to as {\it pure exploration} problem) in linear bandits with finite set of arms has received less attention \cite{soare2014best, karnin2016verification, xu2017, tao18best, fiez2019sequential}, and does not admit a fully satisfactory solution. In the pure exploration problem with fixed confidence, one has to design a $\delta$-PAC algorithm (able to identify the best arm with probability at least $1-\delta$) using as few samples as possible. Such an algorithm consists of a sampling rule (an active policy to sequentially select arms), a stopping rule, and a decision rule that outputs the estimated best arm. The number of rounds before the algorithm stops is referred to as its sample complexity. An instance-specific information-theoretical lower bound of the expected sample complexity has been derived in \cite{soare2015thesis}. However, we are lacking simple and practical algorithms achieving this bound. Importantly, existing algorithms exhibit scalability issues as they always include subroutines that explicitly depend on the number of arms (refer to the related work for details). They may also be computationally involved.

In this paper, we present a new best-arm identification algorithm for linear bandits with finite set of arms, whose sample complexity matches the information-theoretical lower bound. The algorithm follows the track-and-stop principle proposed in \cite{garivier2016optimal} for pure exploration in bandits without structure. Its sampling rule tracks the optimal proportion of arm draws, predicted by the sample complexity lower bound and estimated using the least-squares estimator of the system parameter. Remarkably, this tracking procedure can be made as {\it lazy} as we wish (the estimated optimal proportion of draws can be updated rarely -- not every round) without compromising the asymptotic optimality of the algorithm. The stopping rule of our algorithm is classically based on a generalized likelihood ratio test. However the exploration threshold defining its stopping condition is novel, and critically, we manage to make it independent of the number of arms. Overall our algorithm is simple, scalable, and yet asymptotically optimal. In addition, its computational complexity can be tuned by changing the frequency at which the tracking rule is updated, without affecting its theoretical guarantees.

We also study the pure exploration problem in linear bandits with a continuous set of arms. We restrict our attention to the case where the set of arms consists of the $(d-1)$-dimensional unit sphere. We establish a sample complexity lower bound satisfied by any $(\epsilon,\delta)$-PAC algorithms (such algorithms identify an $\epsilon$-optimal arm with probability at least $1-\delta$). This bound scales as ${d\over \varepsilon}\log(1/\delta)$. We finally propose an algorithm whose sample complexity matches the lower bound order-wise.

%
%
%


{\bf Related work.} Best-arm identification algorithms in linear bandits with a finite set of $K$ arms have been proposed and analyzed in \cite{soare2014best, karnin2016verification, xu2017, tao18best, fiez2019sequential}. Soare et al. \cite{soare2014best} leverage tools from G-optimal experimental design to devise the ${\cal X}{\cal Y}$-adaptive algorithm returning the best arm and with sample complexity $\tau$ satisfying $\tau \lesssim (M^\star\vee T^\star_\mu \log( {K^2}/{\delta}))(\log \log ({K^2}/{\delta} ) + \log({1}/{\Delta^2_{\min} } ))$, w.p. $1-\delta$, where $\mu$ is the parameter defining the reward function, $\Delta_{\min}$ is the minimal gap between the best and a sub-optimal arm, $T^\star_\mu \log(1/\delta)$ is the information theoretical lower bound for the expected sample complexity of $\delta$-PAC algorithms, and where $M^\star$ is an instance-dependent constant. ${\cal X}{\cal Y}$-adaptive runs in phases, and eliminates arms at the end of each phase. The use of phases requires rounding procedures, which come with $d^2$ additional rounds in the sample complexity. The algorithm also requires to solve in each round an optimization problem similar to that leading to the sample complexity lower bound (see \textsection \ref{sec:low}). Improved versions of ${\cal X}{\cal Y}$-adaptive have been proposed in \cite{tao18best, fiez2019sequential}. ALBA \cite{tao18best} relies on a novel estimator for $\mu$ (removing the need of rounding procedures). RAGE \cite{fiez2019sequential} offers an improved sample complexity $\tau \lesssim T^*_\mu \log({1}/{\Delta^2_{\min}})(\log( { K^2 }/{\delta}) +d \log( {1}/{\Delta_{\min}^2}) )$ (slightly simplifying the expression). The aforementioned algorithms are rather complicated, and explicitly use the number $K$ of arms in some of their components: $K$ is present in the arm elimination function in ${\cal X}{\cal Y}$-adaptive, in the phase durations in \cite{tao18best, fiez2019sequential}. Importantly, their sample complexity does not match the information-theoretical lower bound when $\delta$ decreases. There is also no guarantees for their expected sample complexity. 

\cite{karnin2016verification} proposes an algorithm based on an explore-and-verify framework and with an asymptotically optimal sample complexity. The algorithm is not practical, but is the first to demonstrate that the lower bound derived in \cite{soare2015thesis} is achievable. In \cite{xu2017}, the authors present LinGapE, an algorithm, as simple as ours. However, its sampling and stopping rules are both sub-optimal (e.g. the algorithm needs to sample all arms at least once), which in turn leads to weak performance guarantees with a sample complexity satisfying $\tau\lesssim K\log(1/\delta)$.

The algorithm we present is as simple as LinGapE, does not run in phases, does not explicitly use the number of arms in its sampling and stopping rules, and has an asymptotically optimal sample complexity, both almost surely and in expectation. 

We are not aware of any work on best-arm identification in linear bandits with a continuous set of arms. We provide here the first results.  

%
%
%
%
%
%


\section{Model and Objective}

We consider a bandit problem with a set ${\cal A}\subset \mathbb{R}^d$ of arms. In round $t\ge 1$, if the decision maker selects arm $a$, she observes as a feedback a random reward $r_t = \mu^\top a +\eta_t$. $\mu\in \mathbb{R}^d $ is unknown, and $(\eta_t)_{t\ge 1}$ is a sequence of i.i.d. Gaussian random variables, $\eta_t\sim {\cal N}(0,\sigma)$. The objective is to learn the arm $a_\mu^\star$ with the highest expected reward $a_\mu^\star=\arg\max_{a\in {\cal A}} \mu^\top a$. Throughout the paper, we assume that $\mu$ and ${\cal A}$ are such that the best arm $a_\mu^\star$ is unique. We also assume that the set of arms ${\cal A}$ spans $\mathbb{R}^d$.

A best-arm identification algorithm consists of a sampling rule, a stopping rule, and a decision rule. The sampling rule decides which arm $a_t$ is selected in round $t$ based on past observations: $a_t$ is ${\cal F}_{t-1}$-measurable, where ${\cal F}_t$ is the $\sigma$-algebra generated by $(a_1,r_1,\ldots,a_t,r_t)$. The stopping rule decides when to stop sampling, and is defined by $\tau$, a stopping time w.r.t. the filtration $({\cal F}_t)_{t\ge 1}$. The decision rule outputs a guess $\hat{a}_\tau$ of the best arm based on observations collected up to round $\tau$, i.e., $\hat{a}_\tau$ is ${\cal F}_\tau$-measurable. The performance of an identification algorithm is assessed through its probabilistic guarantees, and through its sample complexity $\tau$. We consider different probabilistic guarantees, depending on whether the set of arms ${\cal A}$ is finite or continuous. Specifically: for $\epsilon, \delta>0$,

 \begin{definition} [Finite set of arms ${\cal A}$] An algorithm is $\delta$-PAC if for all $\mu$, $\mathbb{P}_\mu[\hat{a}_\tau \neq a_\mu^\star]\le \delta$ and $\mathbb{P}_\mu[\tau<\infty]=1$.
 \end{definition}

\begin{definition} [Continuous set of arms ${\cal A}$] An algorithm is $(\varepsilon,\delta)$-PAC if for all $\mu$, \\ $\mathbb{P}_\mu[\mu^\top (a_\mu^\star -\hat{a}_\tau)> \varepsilon ]\le \delta$ and $\mathbb{P}_\mu[\tau<\infty]=1$.
 \end{definition}

When the set of arms ${\cal A}$ is finite (resp. continuous), the objective is to devise a $\delta$-PAC (resp. $(\varepsilon,\delta)$-PAC) algorithm with minimal expceted sample complexity $\mathbb{E}_{\mu}[\tau]$.

{\bf Notation.} Let $[K]=\{1,\ldots,K\}$. $\Lambda=\{x\in [0,1]^K:\sum_k x_k=1\}$ denotes the simplex in dimension $K$. For $a,b\in [0,1]$, $\skl(a,b)$ is the KL divergence between two Bernoulli distributions of respective means $a$ and $b$. For any $w, w'\in \mathbb{R}^K$, we denote $d_\infty(w, w') = \max_{a \in [K] } \vert w_a -  w'_a \vert$, and for any compact set $C\subset \mathbb{R}^K$, $d_\infty(w, C) = \min_{w' \in C} d_\infty(w, w')$. For $w\in \mathbb{R}^K$, $\textrm{supp}(w)=\{a\in [K]: w_a\neq 0\}$ denotes the support of $w$. $\mathbb{P}_\mu$ (resp. $\mathbb{E}_\mu$) denotes the probability measure (resp. expectation) of observations generated under $\mu$; in absence of ambiguity, we simply use $\mathbb{P}$ (resp. $\mathbb{E}$). For two functions $f$ and $g$, we write $f \lesssim g$ iff there exists a universal constant $C$ such that for all $x$, $f(x)\le Cg(x)$.


\section{Finite set of arms}

Consider a finite set ${\cal A}$ of $K$ arms. We first recall existing lower bounds on the expected complexity of $\delta$-PAC algorithms, and then present our algorithm along with an analysis of its sample complexity.

\subsection{Sample complexity lower bound}\label{sec:low}

Soare \cite{soare2015thesis} derived the following sample complexity lower bound, using the method developed by Garivier and Kaufmann \cite{kaufmann2016complexity} in the case of bandits without structure.

\begin{theorem}\label{thm:lowerbound}
The sample complexity of any $\delta$-PAC algorithm satisfies: $\forall \mu$, $\E_\mu [\tau]\ge \sigma^2 T_\mu^\star \skl(\delta,1-\delta)$, where
\begin{equation}
(T_\mu^\star)^{-1} = \sup_{w \in \Lambda} \min_{a \in \mathcal{A}\backslash a^\star_\mu }\frac{(\mu^\top (a_\mu^\star - a))^2}{ 2(a^\star_\mu - a)^\top \left(\sum_{a \in \mathcal{A}} w_a a a^\top\right)^{-1}  (a^\star_\mu - a) }.
  \end{equation}
\end{theorem}

In the above lower bound, $w$ may be interpreted as the proportions of arm draws, also referred to as {\it allocation}. For $a\in {\cal A}$, $w_a$ represents the fraction of rounds where arm $a$ is selected. This interpretation stems from the proof of Theorem \ref{thm:lowerbound}, where $w_a=\mathbb{E}_\mu[N_a(\tau)]/\mathbb{E}_\mu[\tau]$ and $N_a(t)$ is the number of times $a$ is selected up to and including round $t$  (see \cite{soare2015thesis}). The lower bound is obtained by taking the supremum over $w$, i.e., over the best possible allocation.

A different way to define $(T_\mu^\star)^{-1}$ is $\sup_{w \in \Lambda} \psi(\mu,w)$ (a convex program) \cite{soare2015thesis}, where
 \begin{equation}
    \psi(\mu, w) = \min_{\lbrace \lambda: \exists a \neq a^\star_\mu, \lambda^\top (a^\star - a) < 0 \rbrace } \frac{1}{2}(\mu - \lambda)^{\top} \left(\sum_{a \in \mathcal{A}} w_a a a^\top \right)(\mu - \lambda).
  \end{equation}

The next lemmas, proved in Appendix B, confirm that the two definitions of $T_\mu^\star$ are equivalent, and provide useful properties of the function $\psi$ and of its maximizers.

\begin{lemma} \label{lem:continuity0} We have:
 \begin{equation}
    \psi(\mu, w) = \begin{cases} \min_{a \in \mathcal{A}\backslash a^\star_\mu }\frac{\langle \mu, a_\mu^\star - a\rangle^2}{2 (a^\star_\mu - a)^\top \left(\sum_{i=1}^K w_i a_i a_i^\top\right)^{-1}  (a^\star_\mu - a)} & \qquad \textrm{if } \sum_{a \in \mathcal{A}} w_a a a^\top \succ 0, \\
    0 & \qquad \textrm{otherwise}.
  \end{cases}
  \end{equation}
In addition, $\psi$ is continuous in both $\mu$ and $w$, and $w\mapsto \psi(\mu,w)$ attains its maximum in $\Lambda$ at a point $w_\mu^\star$ such that $\sum_{a \in \mathcal{A}} (w_\mu^\star)_a a a^\top$ is invertible.
 \end{lemma}

 \begin{lemma} \label{lem:continuity}
 (Maximum theorem) Let $\mu \in \mathbb{R}^d$ such that $a^\star_\mu$ is unique. Define $\psi^*(\mu) = \max_{w \in \Lambda} \psi(\mu, w)$ and $C^\star(\mu) = \argmax_{w \in \Lambda} \psi(\mu, w)$. Then $\psi^\star$ is continuous at $\mu$, and $C^\star(\mu)$ is convex, compact and non-empty. Furthermore, we have\footnote{This statement is that of upper hemicontinuity of a correspondence.} for any open neighborhood $\mathcal{V}$ of $C^\star(\mu)$, there exists an open neighborhood $\mathcal{U}$ of $\mu$, such that for all $\mu' \in \mathcal{U}$, we have $C^\star(\mu') \subseteq \mathcal{V}$.
\end{lemma}

\subsection{Least-squares estimator}

Our algorithm and its analysis rely on the least-squares estimator of $\mu$ and on its performance. This estimator $\hat{\mu}_t$ based on the observations in the $t$ first rounds is: $\hat{\mu}_t = (\sum_{s=1}^t a_s a_s^\top)^{\dagger} (\sum_{s=1}^t a_s r_s)$. The following result provides a sufficient condition on the sampling rule for the convergence of $\hat{\mu}_t$ to $\mu$. This condition depends on the asymptotic spectral properties of the covariates matrix $\sum_{s=1}^t a_s a_s^\top$. We also provide a concentration result for the least-squares estimator. Refer to Appendix C for the proofs of the following lemmas.

\begin{lemma}\label{lem:ls} Assume that the sampling rule satisfies $\liminf_{t\to \infty} \lambda_{\min}\left(\frac{1}{t^\alpha} \sum_{s=1}^t a_s a_s^\top \right)  > 0$ almost surely (a.s.), for some $\alpha\in (0,1)$. Then, $\lim_{t\to\infty}\hat{\mu}_t=\mu$ a.s.. More precisely, for all $\beta\in (0,\alpha/2)$, $\Vert \hat{\mu}_t - \mu \Vert = o(t^\beta)$ a.s..
\end{lemma}

\begin{lemma}\label{lem:ls concentration}
Let $\alpha>0$ and $L=\max_{a \in \mathcal{A}}\Vert a \Vert$. Assume that $\lambda_{\min}(\sum_{s=1}^t a_s a_s^\top) \ge c t^\alpha$ a.s. for all $t \ge t_0$ for some $t_0\ge 0$ and some constant $c>0$. Then
\begin{equation}
  \forall t \ge t_0 \qquad \mathbb{P}\left(\Vert \hat{\mu}_t - \mu \Vert \ge \varepsilon \right) \le (c^{-1/2}L)^{d} t^{\frac{(1-\alpha)d}{2}} \exp\left( - \frac{c\varepsilon^2 t^{\alpha}}{4\sigma^2}\right).
\end{equation}
\end{lemma}

The least-squares estimator is used in our decision rule. After the algorithm stops in round $\tau$, it returns the arm $\hat{a}_\tau\in \arg\max_{a\in {\cal A}} \hat{\mu}_\tau^\top a$.

\subsection{Sampling rule}

To design an algorithm with minimal sample complexity, the sampling rule should match optimal proportions of arm draws, i.e., an allocation in the set $C^\star(\mu)$. Since $\mu$ is unknown, our sampling rule will track, in round $t$, allocations in the plug-in estimate $C^\star(\hat{\mu}_t)$. To successfully apply this {\it certainty equivalence} principle, we need to at least make sure that using our sampling rule, $\hat{\mu}_t$ converges to $\mu$. Using Lemma \ref{lem:ls}, we can design a family of sampling rules with this guarantee:

\begin{lemma} \label{lem:sufficient exploration}
(Forced exploration) Let $\mathcal{A}_0 =\{ a_0(1),\ldots,a_0(d)\} \subseteq \mathcal{A}$ : $\lambda_{\min}(\sum_{a \in \mathcal{A}_0}a a^\top ) > 0$.\\ Let $(b_t)_{t\ge 1}$ be an arbitrary sequence of arms. Furthermore, define for all $t \ge 1$,
$
f(t) = c_{\mathcal{A}_0}\sqrt{t}
$
where $c_{\mathcal{A}_0} = \frac{1}{\sqrt{d}}\lambda_{\min}\left(\sum_{a \in \mathcal{A}_0}a a^\top\right)$.
Consider the sampling rule, defined recursively as: $i_0=1$, and for $t\ge 0$, $i_{t+1} = (i_{t} \mod d) + \indicator_{\left\lbrace \lambda_{\min}\left( \sum_{s=1}^t a_s a_s^\top \right) < f(t)\right\rbrace}$ and
\begin{align}\label{eq:forced}
a_{t+1} & = \begin{cases} a_0(i_t)  & \hbox{if} \quad \lambda_{\min}\left( \sum_{s=1}^t a_s a_s^\top \right) < f(t), \\
b_t & \hbox{otherwise.}
\end{cases}
\end{align}
Then for all $t\ge \frac{5d}{4} + \frac{1}{4d} + \frac{3}{2}$, we have
$
\lambda_{\min}\left(\sum_{s=1}^t a_s a_s^\top \right) \ge f(t - d -1).
$
\end{lemma}

A sampling rule of the family defined in Lemma \ref{lem:sufficient exploration} is forced to explore an arm in ${\cal A}_0$ (in a round robin manner) if $\lambda_{\min}( \sum_{s=1}^t a_s a_s^\top )$ is too small. According to Lemma \ref{lem:ls}, this forced exploration is enough to ensure that $\hat{\mu}_t$ converges to $\mu$ a.s.. Next in the following tracking lemma, we show how to design the sequence $(b_t)_{t\ge 1}$ so that the sampling rule gets close to a set $C$ we wish to track.

\begin{lemma}\label{lem:tracking lemma}
(Tracking a set $C$)  Let $(w(t))_{t\ge1}$ be a sequence taking values in $\Lambda$, such that there exists a compact, convex and non empty subset $C$ in $\Lambda$, there exists $\varepsilon >0$ and $t_0(\varepsilon) \ge 1$ such that
$
\forall t \ge t_0, d_\infty(w(t), C) \le \varepsilon.
$\\
Define for all $a\in \mathcal{A}$, $N_a(0) = 0$. Consider a sampling rule defined by (\ref{eq:forced}) and
\begin{equation}\label{eq:track1}
b_t = \argmin_{a \in \textrm{supp}(\sum_{s=1}^t w(s))} \left( N_a(t) - \sum_{s=1}^t w_a(t)\right),
\end{equation}
where $N_a(0) = 0$ and for $t\ge 0$, $N_a(t+1)=N_a(t) +\indicator_{\{ a_t=a\}}$. \\
Then there exists $t_1(\varepsilon) \ge t_0(\varepsilon)$ such that $\forall t \ge t_1(\varepsilon)$, $d_\infty( (N_a(t)/t)_{a\in \mathcal{A}} , C) \le (p_t + d - 1) \varepsilon$
  where $p_t = \vert \textrm{supp}(\sum_{s=1}^t w(s)) \backslash \mathcal{A}_{0} \vert \le K - d $.
\end{lemma}

{\bf The lazy tracking rule.} The design of our tracking rule is completed by choosing the sequence $(w(t))_{t\ge 1}$ in (\ref{eq:track1}). The only requirement we actually impose on this sequence is the following condition: there exists a non-decreasing sequence $(\ell(t))_{t\ge 1}$ of integers with $\ell(1) = 1$, $\ell(t) \le t-1$ for $t>1$ and $\lim_{t\to\infty}\ell(t)=\infty$ and such that
\begin{equation}\label{eq:lazycond}
  \lim_{t \to \infty} \min_{s\ge \ell(t)} d_\infty(w(t), C^\star(\hat{\mu}_s) )= 0. \qquad a.s..
\end{equation}

This condition is referred to as the {\it lazy} condition, since it is very easy to ensure in practice. For example, it holds for the following lazy tracking rule. Let ${\cal T}=\{ t_n: n\ge 1\}$ be a deterministic increasing set of integers such that $t_n\to \infty$ as $n\to\infty$, we can define $(w(t))_{t\ge 1}$ such that it tracks $C^\star(\hat{\mu}_t)$ only when $t\in {\cal T}$. Specifically, if $t\in {\cal T}$, $w(t+1)\in C^\star(\hat{\mu}_t)$, and $w(t+1)=w(t)$ otherwise. For this sequence, (\ref{eq:lazycond}) holds with $\ell(t)=t-1$. The lazy condition is sufficient to guarantee the almost sure asymptotical optimality of the algorithm. To achieve optimality in terms of expected sample complexity, we will need a slightly stronger condition, also easily satisfied under some of the above lazy tracking rules, see details in \textsection \ref{subsec:sc}.

The following proposition states that the lazy sampling rule is able to track the set $C^*(\mu)$. It follows from the fact that $\hat{\mu}_t$ converges to $\mu$ (thanks to Lemmas \ref{lem:ls} and \ref{lem:sufficient exploration}) and from combining the maximum theorem (Lemma \ref{lem:continuity}) and Lemma \ref{lem:tracking lemma}. All proofs related to the sampling rule are presented in Appendix E.

\begin{proposition}\label{prop:tracking}
Under any sampling rule (\ref{eq:forced})-(\ref{eq:track1}) satisfying the lazy condition (\ref{eq:lazycond}), the proportions of arm draws approach $C^\star(\mu)$: $\lim_{t\to\infty}
d_{\infty} ( (N_a(t) /t)_{a \in \mathcal{A}}, C^\star(\mu)) =0$, a.s..
\end{proposition}

\subsection{Stopping rule}

We use the classical Chernoff's stopping time. Define the generalized log-likelihood ratio for all pair of arms $a,b \in \mathcal{A}$, $t\ge 1$, and $\varepsilon\ge0$ as
\begin{equation*}
  Z_{a,b,\varepsilon}(t) = \log\left(  \frac{\max_{\lbrace \mu: \mu^\top(a - b) \ge  -\varepsilon\rbrace} f_\mu(r_t,a_t, \dots, r_1,a_1)}{\max_{\lbrace \mu: \mu^\top(a - b)\le -\varepsilon\rbrace}f_\mu(r_t,a_t, \dots, r_1,a_1)} \right),
\end{equation*}
where $f_\mu(r_t,a_t, \dots, r_1,a_1) \propto \exp (- \frac{1}{2}\sum_{s=1}^t (r_s - \mu^\top a_s)^2 )$ under our Gaussian noise assumption. We may actually derive an explicit expression of $Z_{a,b,\varepsilon}(t)$ (see Appendix D for a proof):

\begin{lemma}\label{lem:GLLR closed form}
Let $t\ge 1$ and assume that $\sum_{s=1}^t a_s a_s^\top \succ 0$. For all $a,b \in  \mathcal{A}$, we have:
\begin{align*}
  Z_{a,b,\varepsilon}(t) & =
    \textrm{sgn}(\hat{\mu}_t^\top (a-b) + \varepsilon)\frac{(\hat{\mu}_t^\top (a-b) +\varepsilon)^2}{2(a-b)^\top\left(\sum_{s=1}^t a_s a_s^\top\right)^{-1} (a-b)}.
\end{align*}
\end{lemma}

Here we use $Z_{a,b}(t)=Z_{a,b,0}(t)$ ($Z_{a,b,\varepsilon}$ will be used in the case of continuous set of arms). Note that $Z_{a,b}(t) \ge 0$ iff $a \in \argmax_{a\in\mathcal{A}} \hat\mu_t^\top a$. Denoting $Z(t)=\max_{a \in \mathcal{A}} \min_{b \in \mathcal{A}\backslash a} Z_{a,b}(t)$, the stopping rule is defined as follows:
\begin{equation}\label{eq:stop}
 \tau =  \inf \left\lbrace  t \in \mathbb{N}^* \; : Z(t) > \beta(\delta, t) \textrm{ and } \sum_{s=1}^t a_s a_s^\top \succeq c I_d \right\rbrace
\end{equation}
where $\beta(\delta, t)$ is referred to as the {\it exploration threshold} and $c$ is some positive constant (refer to Remark \ref{rem:c} for a convenient choice for $c$). The exploration threshold $\beta(\delta,t)$ should be chosen so that the algorithm is $\delta$-PAC. We also wish to design a threshold that does not depend in the number $K$ of arms. These requirements leads to the exploration threshold defined in the proposition below (its proof is presented in Appendix D and relies on a concentration result for self-normalized processes \cite{abbasi2011improved}).

\begin{proposition}\label{prop:stopping rule}
  Let $u>0$, and define:
  \begin{equation}\label{eq:rate}
    \beta(\delta, t) = (1+u) \sigma^2 \log\left(\frac{\det\left((uc)^{-1} \sum_{s=1}^t a_s a_s^\top +  I_d\right)^{\frac{1}{2}}}{\delta}\right).
  \end{equation}
  Under any sampling rule, and a stopping rule (\ref{eq:stop}) with exploration rate (\ref{eq:rate}), we have: $\mathbb{P}\left( \tau < \infty, \mu^\top (a_\mu^\star - \hat{a}_\tau) > 0  \right) \le \delta$.
\end{proposition}

The above proposition is valid for any sampling rule, but just ensures that 'if' the algorithm stops, it does not make any mistake w.p. $1-\delta$. To get a $\delta$-PAC algorithm, we need to specify the sampling rule.

\begin{remark}\label{rem:c} (Choosing $c$ and $u$) A convenient choice for the constant $c$ involved in (\ref{eq:stop}) and (\ref{eq:rate}) is $c = \max_{a \in \mathcal{A}}\| a \|^2$. With this choice, we have: $\det (c^{-1} \sum_{s=1}^t a_s a_s^\top + I _d ) \le ( t + 1)^d$.
  The constant $u$ should be chosen so that the threshold in \eqref{eq:rate} is lowered for instance one my choose $u=1$. From these choices the threshold can be as simple as $\beta(\delta,t) = 2\sigma^2 \log ( t^{{d\over 2}} / \delta)$. In addition, if we use a sampling rule with forced exploration as in (\ref{eq:forced}), then in view of Lemma \ref{lem:sufficient exploration}, the second stopping condition $\sum_{s=1}^t a_s a_s^\top \succeq c I_d$ is satisfied as soon as $t$ exceeds $d+1+{c^2d\over \lambda_{\min}(\sum_{a\in {\cal A}_0}aa^\top)}$.
\end{remark}

\subsection{Sample complexity analysis}\label{subsec:sc}

\begin{algorithm}[t]
\SetAlgoLined
\KwIn{Arms $\mathcal{A}$, confidence level $\delta$, set ${\cal T}$ of lazy updates}
 {\bf Initialization:}  $t=0$, $i = 0$, $A_0 = 0$, $Z(0)=0$, $N(0)=(N_a(0))_{a\in {\cal A}}=0$\;
 \DontPrintSemicolon
 \While{($\lambda_{\min}(A_t) < c)$ or $(Z(t) < \beta(\delta, t)$)}{
 \eIf{$\lambda_{\min}( A_{t}) < f(t)$}{
   $a\leftarrow a_0(i+1)$, $i\leftarrow (i+1 \mod d)$\;
   }{
   $a\leftarrow \argmin_{b \in \textrm{supp}(\sum_{s=1}^t w(s))} \left( N_b(t) - \sum_{s=1}^t w_b(t)\right),$ }
 $t\leftarrow t+1$, sample arm $a$ and update $N(t)$, $\hat{\mu}_t$, $Z(t)$, $A_t\leftarrow A_{t-1}+aa^\top$, $w(t) \leftarrow w(t-1)$\\
 \textbf{if} $t \in \mathcal{T}$ \textbf{then} $w(t) = \argmax_{w \in \Lambda} \psi(\hat{\mu}_t,w)$ \;
  }
 \Return {$\hat{a}_\tau = \argmax_{a\in\mathcal{A}} \hat{\mu}_{\tau}^\top a$}
 \caption{Lazy Track-and-Stop (LTS)}
 \label{algo:LTS}
\end{algorithm}
In this section, we establish that combining a sampling rule (\ref{eq:forced})-(\ref{eq:track1}) satisfying the lazy condition (\ref{eq:lazycond}) and the stopping rule (\ref{eq:stop})-(\ref{eq:rate}), we obtain an asymptotically optimal algorithm. Refer to Appendix F for proofs. An example of such algorithm is the Lazy Track-and-Stop (LTS) algorithm, whose pseudo-code is presented in Algorithm \ref{algo:LTS}. LTS just updates the tracking rule in rounds in a set ${\cal T}$.

\begin{theorem}\label{th:sample1}
(Almost sure sample complexity upper bound) An algorithm defined by (\ref{eq:forced})-(\ref{eq:track1})-(\ref{eq:stop})-(\ref{eq:rate}) with a lazy sampling rule (satisfying (\ref{eq:lazycond})) is $\delta$-PAC. Its sample complexity verifies:
$$
\mathbb{P}( \limsup_{\delta \to 0} \frac{\tau}{\log(\frac{1}{\delta})}  \lesssim \sigma^2 T^*_\mu) = 1.
$$
\end{theorem}

To obtain an algorithm with optimal expected sample complexity, we need to consider lazy tracking rules that satisfy the following condition: there exist $\alpha>0$ and a non-decreasing sequence $(\ell(t))_{t\ge 1}$ of integers with $\ell(1) = 1$, $\ell(t) \le t$ and $\lim\inf_{t\to\infty}\ell(t)/t^\gamma >0$ for some $\gamma>0$ and such that
\begin{equation}\label{eq:lazycond2}
  \forall \varepsilon>0, \ \ \exists h(\varepsilon)\ \  : \ \ \forall t\ge 1, \ \  \mathbb{P}\left(\min_{s \ge \ell(t)} d_\infty(w(t), C^\star(\hat{\mu}_s)) > \varepsilon  \right) \le \frac{h(\varepsilon) }{t^{2+\alpha}}.
\end{equation}

The condition (\ref{eq:lazycond2}) is again easy to ensure. Assume that we update $w(t)$ only if $t\in {\cal T}=\{ t_n: n\ge 1\}$, where $t_n$ is increasing sequence of integers such that $t_n\to\infty$ as $n\to\infty$. Then (\ref{eq:lazycond2}) holds for the sequence $(\ell(t))_{t\ge 1}$ such that $\ell(t_{i+1})=t_i$ for all $i$, provided that $\lim\inf_{n\to\infty}t_{n+1}/t_n^\gamma >0$ for some $\gamma>0$. Examples include: (i) periodic updates of $w(t)$: ${\cal T}=\{1+kP, k\in\mathbb{N} \}$ and $\ell(t) = \max\{1, t-P\}$; (ii) exponential updates ${\cal T}=\{2^k, k\in\mathbb{N} \}$ and $\ell(t)=\max\{1,\lfloor t/2\rfloor\}$. The condition (\ref{eq:lazycond2}) may seem too loose, but we have to keep in mind that in practice, the performance of the algorithm will depend on the update frequency of $w(t)$. However for asymptotic optimality, (\ref{eq:lazycond2}) is enough (the key point is to have some concentration of $\hat{\mu}_t$ around $\mu$, which is guaranteed via the forced exploration part of the sampling rule).

%

\begin{theorem}\label{th:sample2}
(Expected sample complexity upper bound) An algorithm defined by (\ref{eq:forced})-(\ref{eq:track1})-(\ref{eq:stop})-(\ref{eq:rate}) with a sampling rule satisfying (\ref{eq:lazycond}) and (\ref{eq:lazycond2}) is $\delta$-PAC. Its sample complexity verifies:
\begin{equation*}
\limsup_{\delta \to 0}\frac{ \E\left[ \tau \right]}{\log\left(\frac{1}{\delta}\right)} \lesssim \sigma^2  T^\star_\mu.
 \end{equation*}
 \end{theorem}


\section{Continuous set of arms}

We now investigate the case where ${\cal A}=S^{d-1}$ is the $(d-1)$-dimensional unit sphere. Without loss of generality, we restrict our attention to problems where $\mu\in {\cal M}(\varepsilon_0)=\{\eta: \eta^\top a_\eta^\star > \varepsilon_0\}$ for some $\varepsilon_0>0$. The results of this section are proved in Appendix G.

\subsection{Sample complexity lower bound}

\begin{theorem}\label{th:low2} Let $\varepsilon \in (0, \varepsilon_0/5)$, and $\delta \in (0,1)$. The sample complexity of any $(\delta,\varepsilon)$-PAC algorithm satisfies: for all $\mu\in {\cal M}(\varepsilon_0)$, $\E_\mu[\tau] \ge \frac{\sigma^2(d-1)}{20\Vert \mu \Vert \varepsilon} \skl(\delta,1-\delta)$.
\end{theorem}

The above theorem is obtained by first applying the classical change-of-measure argument (see e.g. Lemma 19 \cite{kaufmann2016complexity}). Such an argument implies that under any $(\varepsilon,\delta)$-PAC algorithm, for all {\it confusing} $\lambda$ such that $\{ a \in S^{d-1}: \mu^\top(a_\mu^\star - a) \le \varepsilon \}$ and $\{ a \in S^{d-1}: \lambda^\top(a_\lambda^* - a) \le \varepsilon\}$ are disjoint,
  \begin{equation*}
    (\mu-\lambda)^\top \E\left[\sum_{s=1}^{\tau} a_s a_s^\top \right](\mu - \lambda) \ge 2\skl(\delta,1-\delta).
  \end{equation*}
We then study the solution of the following max-min problem: $\max_{(a_t)_{t\ge 1}}\min_{\lambda \in B_\varepsilon(\mu)}(\mu-\lambda)^\top \E\left[\sum_{s=1}^{\tau} a_s a_s^\top \right](\mu - \lambda)$, where $B_\varepsilon(\mu)$ denotes the set of confusing parameters. The continuous action space makes this analysis challenging. We show that the value of the max-min problem is smaller than $\mathbb{E}_\mu[\tau]{10\|\mu\| \varepsilon\over \sigma^2(d-1)}$, which leads to the claimed lower bound.

\subsection{Algorithm}

We present a simple algorithm whose sample complexity approach our lower bound. We describe its three components below. The decision rule is the same as before, based on the least-squares estimator of $\mu$: $\hat{a}_t \in \argmax_{a \in \mathcal{A}} \hat{\mu}_t^\top a$.

{\bf Sampling rule.} Let $\mathcal{U} = \lbrace u_1, u_2, \dots, u_d \rbrace $ be subset of $S^{d-1}$ that form an orthonormal basis of $\mathbb{R}^d$. The sampling rule just consists in selecting an arms from ${\cal U}$ in a round robin manner: for all $t \ge 1$,  $a_t = u_{(t \mod d)}$.


%
%

{\bf Stopping rule.} As for the case of finite set of arms, the stopping rule relies on a generalized loglikelihood ratio test. Define $Z(t)=  \inf_{\lbrace b \in \mathcal{A}: \vert \hat{\mu}_t^\top (\hat{a}_t-b) \vert \ge \varepsilon_t \rbrace} Z_{\hat{a}_t,b,\varepsilon_t}(t)$, where an expression of $Z_{\hat{a}_t,b,\varepsilon_t}(t)$ is given in Lemma \ref{lem:GLLR closed form}. We consider the following stopping time:
 \begin{equation}\label{eq:stopc1}
 \tau = \inf \left\lbrace t \in\mathbb{N}^*: Z(t) \ge \beta(\delta, t) \textrm{  and  }   \lambda_{\min}\left( \sum_{s=1}^t a_s a_s^\top \right) \ge \max\bigg\{ c, {\rho(\delta, t)\over \Vert\hat{\mu}_t \Vert^2}\bigg\} \right\rbrace.
  \end{equation}
Hence compared to the case of finite set of arms, we add a stopping condition defined by the threshold $\rho(\delta,t)$ and related to the spectral properties of the
covariates matrix.

\begin{proposition}\label{prop:stopping rule2}
Let $(\delta_t)_{t \ge 1}, (\varepsilon_t)_{t \ge 1}$ be two sequences with values in $(0 , 1)$ and $(0, \varepsilon)$, respectively, and such that $\sum_{t=1}^\infty \delta_t < \delta$, and $\lim_{t\to\infty}\varepsilon_t =\varepsilon$. Let $\zeta_t = \log({2\det\left(c^{-1} \sum_{s=1}^t a_s a_s^\top +  I_d\right)^{\frac{1}{2}}})-\log({\delta_t})$, and define:
\begin{align}\label{eq:stopc2}
\beta(\delta, t) &= 2\sigma^2 \zeta_t \ \  \hbox{ and }\ \  \rho(\delta, t)  = \frac{4\sigma^2 \varepsilon_t^2 \zeta_t}{(\varepsilon - \varepsilon_t)^2}
\end{align}
Then under the stopping rule (\ref{eq:stopc1})-(\ref{eq:stopc2}), we have:
  $
  \mathbb{P}_\mu \left( \tau< \infty, \mu^\top (a^\star_\mu - \hat{a}_\tau) > \varepsilon \right) \le \delta.
  $
\end{proposition}

\subsection{Sample complexity analysis}

Under specific choices for the sequence $(\varepsilon_t)_{t \ge 1}$, we can analyze the sample complexity of our algorithm, and show its optimality order-wise.

\begin{theorem}\label{thm:sc-cont}
Choose in the stopping rule $\varepsilon_t = \varepsilon \big( 1 +  \varepsilon (  4\sigma^2 \log( \frac{4}{\delta_t} \left\lceil  \frac{t}{d}\right\rceil  ) )^{-1/2} \big)^{-1}$ (observe that $\varepsilon_t < \varepsilon$ and $\lim_{t\to\infty}\varepsilon_t =\varepsilon$). Then under the aforementioned sampling rule, and the stopping rule (\ref{eq:stopc1})-(\ref{eq:stopc2}), we have : $   \mathbb{P}\left(\limsup_{\delta \to 0} \frac{\tau}{\log(1/\delta)} \lesssim \frac{\sigma^2 d}{\Vert \mu\Vert \varepsilon }\right) = 1$ and $ \limsup_{\delta \to 0} \frac{\E[\tau]}{\log(1/\delta)} \lesssim \frac{\sigma^2 d }{\Vert \mu \Vert \varepsilon}$.
\end{theorem}




\section{Experiments}

We present here a few experimental results comparing the performance of our algorithm to that of RAGE, the state-of-the-art algorithm \cite{fiez2019sequential}, in the case of finite set of arms. We compare Lazy TS and RAGE only because they outperform other existing algorithms. Further experimental results can be found in Appendix A.

{\bf Experimental set-up.} We use the following toy experiment which corresponds to the many arms example in \cite{fiez2019sequential}. $d=2$ and ${\cal A}=\{(1,0),e^{j3\pi/4},e^{j(\pi/4+\phi_i)}, i\in [n-2]\}\subset \mathbb{C}$ where $(\phi_i)$ are i.i.d. $\sim {\cal N}(0,0.09)$. $\mu=(1,0)$. Experiments are made with the risk $\delta= 0.05$.

\paragraph{Implementation of Lazy TS.} To update the allocation $w(t)$, we use Frank Wolf algorithm (without any rounding procedure). At each update, the previous allocation is fed as an initial value for the new optimization problem. We implement the exponential lazy update scheme $\mathcal{T} = \{2^k, k\in\mathbb{N} \}$. The parameters of our stopping rule are $c= c_{\mathcal{A}_0}\sqrt{d}$ (so that after $d$ steps the second condition of the stopping rule is satisfied) and $u=1$; we use the threshold $\beta(6\delta /\pi^2t^2,t)$. The initial exploration matrix $\mathcal{A}_0$ is chosen at random. We implemented two versions of Lazy TS. The first one does not track the average but only the current allocation $w(t)$: $a\leftarrow\argmin_{b\in \textrm{supp}(w(t))} (N_b(t) - t w_b(t))$. The second version tracks the average allocations as described in Algorithm \ref{algo:LTS}.

We further compare our results to that of the Oracle algorithm proposed by \cite{soare2014best}. The algorithm samples from a true optimal allocation $w^\star\in C^\star(\mu)$, and applies  a stopping rule that depends on $K$.

\paragraph{Results.} From the table below, Lazy TS outperforms RAGE most of the times, and the performance improvement gets higher when the number of arms grows. Lazy TS without averaging shows better performance, than with averaging. In Appendix A, we present results for another version of Lazy TS, with even better performance.

\begin{table}[h!]
\label{tab:results}
\begin{center}
\resizebox{\columnwidth}{!}{
\begin{tabular}{ccccccccc}
\toprule
\multirow{2}{*}{\textbf{Algorithm}} &  \multicolumn{2}{c}{\textbf{Lazy TS (No averaging)}} & \multicolumn{2}{c}{\textbf{Lazy TS}} & \multicolumn{2}{c}{\textbf{RAGE}}   &  \multicolumn{2}{c}{\textbf{Oracle}}\\
& \multicolumn{2}{c}{\footnotesize\textbf{Sample Complexity}} & \multicolumn{2}{c}{\footnotesize\textbf{Sample Complexity}} & \multicolumn{2}{c}{\footnotesize \textbf{Sample Complexity}}   &  \multicolumn{2}{c}{\footnotesize\textbf{Sample Complexity}} \\
\cmidrule(r){1-1}
\cmidrule(r){2-3}
\cmidrule(r){4-5}
\cmidrule(r){6-7}
\cmidrule(r){8-9}
Number of arms & Mean & (Std) &  Mean & (Std) & Mean &  (Std)  & Mean & (Std)   \\
\midrule
$(K=1000)$        & 1206.55 & (42.2)   &  1409 & (57)    &\textbf{1148.45}  & (49.82)   & 476.45  & (40.74)  \\
$(K=2500)$        & \textbf{1253.60}  & (47.70)  & 1404 & (57)     &1440.75 & (149.24)  & 492.15  & (43.88)  \\
$(K=5000)$        & \textbf{1247.05} & (81.07)  &  1401 &  (86)  &  1540.3  & (158.90)  & 515.60   & (47.64)  \\
$(K=7500)$        & \textbf{1296.55} & (76.78)  &  1434 &  (78)   &       1598.0  & (164.60)  & 547.65  & (45.77)  \\
\bottomrule
\end{tabular}%
}
\end{center}
\caption{Results for the many arms experiment \cite{fiez2019sequential}}
\end{table}

\section{Conclusion}

In this paper, we present Lazy TS, an algorithm to solve the best-arm identification problem in stochastic linear bandits. The sampling rule of the algorithm just tracks the optimal allocation predicted by the sample complexity lower bound. Its stopping rule is defined through generalized log-likelihood ratio and an exploration threshold that does not depend on the number of arms, but on the ambient dimension only. Lazy TS is asymptotically optimal: we have guarantees on its sample complexity, both almost surely and in expectation. The first experimental results are very promising, as Lazy TS seems to exhibit a much better sample complexity than existing algorithms. We also provide the first results on the pure exploration problem in the linear bandits with a continuous set of arms. 

The analysis presented in this paper suggests several extensions. We can easily generalize the results to non-Gaussian reward distributions (e.g. bounded, from a one-parameter exponential family). It would be interesting to extend our results in the continuous setting to generic convex sets of arms (we believe that the instance-specific sample complexity lower bound would just depend on the local smoothness of the set of arms around the best arm). A more challenging but exciting question is to derive tight non-asymptotic sample complexity upper bound for Lazy TS, so as to characterize the trade-off between the laziness of the algorithm and its sample complexity.

\newpage
\bibliographystyle{unsrt}
\bibliography{references,bandit}

\newpage

\appendix
\section{Numerical experiments}

This section provides additional numerical results, and comparisons of Lazy TS and RAGE. We actually present the results of a slightly different version of Lazy TS than that considered in the main document (see details below). This new version exhibits much better performance.

\subsection{Experimental set-up}

{\bf The many arms example.} We use the same problems as those reported in the main document. Namely, the following toy experiment that corresponds to the many arms example in \cite{fiez2019sequential}. $d=2$ and ${\cal A}=\{(1,0),e^{j3\pi/4},e^{j(\pi/4+\phi_i)}, i\in [n-2]\}\subset \mathbb{C}$ where $(\phi_i)$ are i.i.d. $\sim {\cal N}(0,0.09)$. $\mu=(1,0)$. Experiments are made with the risk $\delta= 0.05$.

\textbf{Implementation of Lazy TS.} Our implementation for the following results is almost the same as the one described in Section 5. The only difference lies in the stopping rule: we use improved constants when defining the threshold \eqref{eq:rate}. The new constant is $u=0.1$ (before it was set to $1$), and the threshold is $\beta(\delta,t)$ (before we were using $\beta(\delta6/(\pi t)^2,t)$).

All experiments were executed on a stationary desktop computer, featuring an Intel Xeon Silver 4110 CPU, 48GB of RAM. Ubuntu 18.04 was installed on the computer. We set up our experiments using Python 3.7.7. The code is available at the following link \url{https://www.dropbox.com/s/xqj7h7jbw7rb95v/code_lazy_ts.zip?dl=0}.

\subsection{Results}

\textbf{Sample complexity.} The results on the sample complexity are reported in Table 2. Lazy TS with and without averaging significantly outperforms RAGE \cite{fiez2019sequential} and even the Oracle \cite{soare2014best}. At first, it seems surprising that the Oracle is beaten by Lazy TS, but this can be explained as follows. Even if the Oracle is aware, from the beginning, of the optimal sampling rule, its stopping rule is not efficient and depends on the number of arms $K$. The stopping rule in Lazy TS is independent of $K$, and indeed, Lazy TS performance is less sensitive to the number of arms than that of RAGE or the Oracle. The results also suggest that the Lazy TS algorithms with or without averaging perform similarly. As a final note all algorithms ended with success over all simulations.

\textbf{Run-time.} The run-time of Lazy TS and RAGE are reported in Table 3. Overall, both algorithms are efficient. We note that RAGE is slightly faster. However we expect that for extremely large numbers of arms, Lazy TS would run faster than RAGE (the sample complexity of Lazy TS is more resilient to an increase in the number of arms). In Lazy TS, we have used the exponential lazy update scheme with $\mathcal{T}=\lbrace 2^k: k \in\mathbb{N}^*\rbrace$. We believe that by fine-tuning this laziness, we would be able achieve a better trade-off between computational efficiency and sample complexity.

\textbf{Support of Lazy TS.} Finally, we study the support of the allocation chosen under Lazy TS. The expected size of the support of Lazy TS on a single run is reported in Table 4. Even if the number of arms $K$ is large (in comparison with the ambient dimension), Lazy TS only tracks allocations that are sparse, i.e. using very few arms. We further note that the averaging scheme in the tracking rule does not really affect the support. This is a nice feature as it could allow for the design of a more memory-efficient algorithm.

\newpage

\begin{table}[h!]
\label{tab:results1}
\begin{center}
\resizebox{\columnwidth}{!}{
\begin{tabular}{ccccccccc}
\toprule
\multirow{2}{*}{\textbf{Algorithm}} &  \multicolumn{2}{c}{\textbf{Lazy TS }} & \multicolumn{2}{c}{\textbf{Lazy TS (No averaging)}} & \multicolumn{2}{c}{\textbf{RAGE}}   &  \multicolumn{2}{c}{\textbf{Oracle}}\\
& \multicolumn{2}{c}{\footnotesize\textbf{Sample Complexity}} & \multicolumn{2}{c}{\footnotesize\textbf{Sample Complexity}} & \multicolumn{2}{c}{\footnotesize \textbf{Sample Complexity}}   &  \multicolumn{2}{c}{\footnotesize\textbf{Sample Complexity}} \\
\cmidrule(r){1-1}
\cmidrule(r){2-3}
\cmidrule(r){4-5}
\cmidrule(r){6-7}
\cmidrule(r){8-9}
Number of arms & Mean & (Std) &  Mean & (Std) & Mean &  (Std)  & Mean & (Std)   \\
\midrule
$(K=1000)$      & \textbf{424.5}   & (29.1)  & \textbf{424.5}  & (29.1)  & 1148.45 & (49.82)   & 476.45  &  (40.7)   \\
$(K=2500)$      & 458.15  & (28.1)  & \textbf{455.95} & (28.3)  & 1440.75 & (149.24)  & 492.15  &   (43.9)  \\
$(K=5000)$      & 434.65  & (32.51) & \textbf{433.6}  & (32.6)  & 1540.3  & (158.9)   & 515.6   &   (47.6)  \\
$(K=7500)$      & 448.0   & (36.9)  & \textbf{447.45} & (36.8)  & 1598.0  & (164.6)   & 547.65  &   (45.8)  \\
$(K=10000)$     & \textbf{452.85}  & (31.6)  & 452.95 & (31.6)  & 1479.4  & (52.0)    & 564.85  &   (46.9) \\
\bottomrule
\end{tabular}%
}
\end{center}
\caption{Sample complexity. Results for the many arms experiment \cite{fiez2019sequential}}
\end{table}

\begin{table}[h!]
\label{tab:results2}
\begin{center}
\begin{tabular}{ccccccccc}
\toprule
\multirow{2}{*}{\textbf{Algorithm}} &  \multicolumn{2}{c}{\textbf{Lazy TS}} & \multicolumn{2}{c}{\textbf{Lazy TS (No averaging)}} & \multicolumn{2}{c}{\textbf{RAGE}}\\
& \multicolumn{2}{c}{\footnotesize\textbf{Run time (s)}} & \multicolumn{2}{c}{\footnotesize\textbf{Run time (s)}} & \multicolumn{2}{c}{\footnotesize \textbf{Rune time (s)}}   \\
\cmidrule(r){1-1}
\cmidrule(r){2-3}
\cmidrule(r){4-5}
\cmidrule(r){6-7}
Number of arms & Mean & (Std) &  Mean & (Std) & Mean &  (Std)   \\
\midrule
$(K=1000)$        & 13.62 & (0.5)     & 13.99   &  (0.5)   & 34.0  & (0.5)        \\
$(K=2500)$        & 90.25 & (2.9)     & 89.41   &  (3.1)   & 156.42  & (1.1)      \\
$(K=5000)$        & 940.97 & (40.4)   & 948.86  &  (40.3)  & 429.67  & (7.47)     \\
$(K=7500)$        & 1340.83 & (61.5)  & 1349.90 &  (61.4)  & 707.09  & (9.47)     \\
$(K=10000)$       & 1893.73 & (79.9)  & 1915.03 &  (80.3)  & 1575.30  & (12.43)   \\
\bottomrule
\end{tabular}%
\end{center}
\caption{Runtime. Results for the many arms experiment \cite{fiez2019sequential}}
\end{table}

\begin{table}[h!]
\label{tab:results3}
\begin{center}
\begin{tabular}{ccccccccc}
\toprule
\multirow{2}{*}{\textbf{Algorithm}} &  \multicolumn{2}{c}{\textbf{Lazy TS }} & \multicolumn{2}{c}{\textbf{Lazy TS (No averaging)}}\\
& \multicolumn{2}{c}{\footnotesize\textbf{Support size}} & \multicolumn{2}{c}{\footnotesize\textbf{Support size}}  \\
\cmidrule(r){1-1}
\cmidrule(r){2-3}
\cmidrule(r){4-5}
\cmidrule(r){6-7}
\cmidrule(r){8-9}
Number of arms & Mean & (Std) &  Mean & (Std)    \\
\midrule
$(K=1000)$        & 5.37 & (0.25)   & 2.04  &  (0)  \\
$(K=2500)$        & 5.72 & (0.20)   & 2.03  &  (0)  \\
$(K=5000)$        & 5.41 & (0.20)   & 2.04  &  (0)  \\
$(K=7500)$        & 5.34 & (0.20)   & 2.03  &  (0)  \\
$(K=10000)$       & 5.26 & (0.21)   & 2.04  &  (0)  \\
\bottomrule
\end{tabular}%
\end{center}
\caption{Support size. Results for the many arms experiment \cite{fiez2019sequential}. For the standard deviation, we put $(0)$ when the value is smaller than $10^{-2}$.}
\end{table}

\newpage

\clearpage
\newpage


\newpage
\section{Properties of $\psi$}

\subsection{Proof of Lemma \ref{lem:continuity0}}

Let $(\mu,w) \in \mathbb{R}^d \times \Lambda$ such that $a^\star_\mu$ is unique. For the first part of the claim we refer to the proof of \cite[Theorem 3.1.]{soare2015thesis}. Now let us prove the continuity of $\psi$ at $(\mu,w)$. Consider the set of bad parameters with respect to $\mu$, $B(\mu) \subseteq \mathbb{R}^d$
$$
B(\mu) = \left \lbrace \lambda: \lambda \in \mathbb{R}^d \textrm{ and } \exists a \in \mathcal{A}\backslash \lbrace a^\star_\mu \rbrace \; \lambda^\top( a - a^\star_\mu) > 0 \right\rbrace,
$$
and denote
$$
f(\mu,\lambda, w) = \frac{1}{2}(\mu- \lambda)^\top \left( \sum_{a \in \mathcal{A}}w_a a a^\top \right) (\mu - \lambda).
$$
Let  $(\mu_t, w_t)_{t \ge 1}$ be a sequence taking values in $\mathbb{R}^d \times \Lambda$ and converging to $(\mu, w)$.
 Let $\varepsilon < 1 \wedge \min_{a \in \mathcal{A} \backslash \lbrace a^\star_\mu \rbrace}  \frac{\langle \mu, a^*_\mu - a \rangle}{\Vert a^\star_\mu - a \Vert}$, and let $t_1 \ge 1$ such that for all $t\ge t_1$ we have $\Vert (\mu_t, w_t) - (\mu, w)\Vert < \varepsilon$. Now, by our choice of $\varepsilon$,
 and uniqueness of $a^\star_\mu$ it holds that $B(\mu_t) = B(\mu)$. Furthermore, note that $f(\mu, \lambda, w)$ is a polynomial in $\mu, \lambda, w$, thus it is in inparticular continuous in $\mu, w$, and there exists $t_2 \ge 1$ such that for all $t \ge t_2$ and for all $\lambda \in \mathbb{R}^d$, it holds that $\vert f(\mu_t, \lambda, w_t) - f(\mu, \lambda, \mu_t) \vert \le \varepsilon f(\mu, \lambda, \mu_t)$.
 Hence, with our choice of $\varepsilon$, we have for all $t\ge t_1\vee t_2$
\begin{align*}
    \vert \psi(\mu, w) - \psi(\mu_t, w_t)\vert  &  = \Big \vert \min_{\lambda \in B(\mu)} f(\mu, \lambda, w) - \min_{\lambda \in B(\mu)} f(\mu_t, \lambda, w_t) \Big\vert \\
    & \le  \varepsilon \Big \vert \min_{\lambda \in B(\mu)} f(\mu, \lambda, w) \Big \vert \\
    & \le \varepsilon \vert \psi(\mu, w) \vert.
\end{align*}
This concludes the proof of the continuity of $\psi$.

Now, we know that $w \mapsto \psi(\mu, w)$ is continuous on $\Lambda$, and by compactness of the simplex, the maximum is attained at some $w^\star_\mu \in \Lambda$. Furthermore, since $\mathcal{A}$ spans $\mathbb{R}^d$, we may construct an allocation $\tilde{w}$ such that $\sum_{a \in \mathcal{A}} \tilde{w}_a a a^\top$ is a positive definite matrix.
In addition, by construction of $B(\lambda)$, there exists some $M>0$ such that for all $\lambda \in B(\mu)$ we have  $\Vert \mu - \lambda\Vert > M$, which implies that $\psi(\mu, \tilde{w}) \ge M^2 \lambda_{\min}\left(\sum_{a \in \mathcal{A}} \tilde{w}_a aa^\top\right) > 0$.
On the other for any allocation $w \in \Lambda$ such that $\sum_{a\in\mathcal{A}} w_a a a^\top$ is rank deficient, we may find a $\lambda \in B(\mu)$ where $\lambda -\mu$ is in the null space of $\sum_{a\in\mathcal{A}} w_a a a^\top$. Therefore, $\sum_{a \in \mathcal{A}} (w^\star_\mu)_a a a^\top$ is invertible \ep

\subsection{Proof of Lemma \ref{lem:continuity}}

The lemma is a direct consequence of the maximum theorem (a.k.a. Berge's theorem) \cite{sundaram1996first} and only requires that $\psi$ is continuous in $(\mu, w) \in \mathbb{R}^d \times \Lambda$, that $\Lambda$ is compact, convex and non-empty, and that $\psi$ is concave in $w$ for each $\mu' \in \mathbb{R}^d$ in an open neighberhood of $\mu$. These requirements hold naturally in our setting:
(i) by Lemma \ref{lem:continuity0}, we have for all $\mu\in \mathbb{R}^d$ such that $a^\star_\mu$ is unique and for any $w \in \Lambda$, $\psi$ is continuous in $(\mu,w)$; (ii) $\Lambda$ is a non-empty, compact and convex set; (iii) for all $\mu\in \mathbb{R}^d$, $w \mapsto \psi(\mu, w)$ is concave as it can be expressed as the infimum of linear functions in $w$.
Therefore, the maximum theorem applies and we obtain the desired results. \ep

\newpage
\section{Least Squares Estimator}
In this appendix, we present concentration bounds and convergence statements on the least squares estimator. We may recall that the least squares estimation error $\hat{\mu}_t - \mu$ can be expressed conveniently in the following form\footnote{We mean by $A^{-1}$ the pseudo-inverse of $A$ when the matrix is not invertible.}: $\hat{\mu}_t - \mu = (\sum_{s=1}^t a_s a_s^\top)^{-1}(\sum_{s=1}^t a_s \eta_s)$. To make notations less cluttered, we prefer to express our derivations in matrix form where we define the covariates matrix $A_t = \begin{bmatrix}
  a_1 & \dots & a_t
\end{bmatrix}^\top$ and noise vector $E_t = \begin{bmatrix} \eta_1 & \dots & \eta_t \end{bmatrix}^\top$. We may then write $\hat{\mu}_t - \mu = (A_t^\top A_t)^{-1}(A_t^\top E_t)$. Furthermore, we will reapeatedly use the following decoposition
\begin{equation} \label{eq:ls decomposition}
\Vert \hat{\mu}_t - \mu \Vert  = \Vert (A_t^\top A_t)^{-1} (A_t^\top E_t) \Vert \le \Vert A_t^{\top}E_t \Vert_{(A_t^\top A_t)^{-1}} \Vert (A_t^\top A_t)^{-1/2} \Vert
\end{equation}
where we have $\Vert x \Vert_A = \sqrt{x^\top A x}$ for some semi-definite positive matrix $A$. The above inequality follows from Cauchy-Schwarz inequality. We also observe that when $A_t^\top A_t$ is invertible, we have $\Vert (A_t^\top A_t)^{-1/2} \Vert = \lambda_{\min}(A_t^\top A_t)^{-1/2}$.

\subsection{Self-Normalized processes}
We first present convenient tools from the theory of self-normalized processes \cite{pena2008self}, namely the deviation bounds established by Abbasi-Yadkouri et al. in \cite{abbasi2011improved}.
\begin{proposition}[Theorem 1. in \cite{abbasi2011improved}]\label{prop:self concentration}
  Let $ (\mathcal{F}_t )_{t \ge 0} $ be a filtration. Let $\lbrace \eta_t \rbrace_{t \ge 1}$ be a real-valued stochastic process such that for all $t \ge 1$, $\eta_t$ is $\mathcal{F}_{t-1}$-measurable and satisfies with some postive $\sigma$, the conditional $\sigma$-sub-gaussian condition:
  $
  \E\left[\exp(x \eta_t) \vert \mathcal{F}_{t-1} \right] \le \exp\left (- x^2\sigma^2/2\right),
  $ for all $x \in \mathbb{R}$.
  Let $(a_t )_{t \ge 1}$ be an $\mathbb{R}^d$-valued stochastic process adapted to $\lbrace \mathcal{F}_t \rbrace_{t \ge 0}$. Furthermore, let $V$ be a positive definite matrix. Then for all $\delta \in (0,1)$ we have
  \begin{equation*}
  \mathbb{P}\left( \left\Vert A_t^\top E_t \right\Vert_{(A_t^\top A_t + V)^{-1}}^2 \le
   2\sigma^2 \log\left( \det\left( (A_t^\top A_t + V)V^{-1}\right)^\frac{1}{2} \big/ \delta \right)  \right) \ge 1 - \delta.
 \end{equation*}
\end{proposition}

The following result is a stronger version of Proposition \ref{prop:self concentration} and in fact is behind its proof.
\begin{proposition}[Lemma 9. in \cite{abbasi2011improved}]\label{prop:self concentration stopping}
With the same assumptions as in the above proposition. Let $\tau$ be any stopping time with respect to the filtration $(\mathcal{F})_{t\ge 1}$. Then, for $\delta > 0$, we have
\begin{equation*}
\mathbb{P}\left( \left\Vert A_\tau^\top E_\tau \right\Vert_{(A_\tau^\top A_\tau + V)^{-1}}^2 \le
 2\sigma^2 \log\left( \det\left( (A_\tau^\top A_\tau+ V)V^{-1}\right)^\frac{1}{2} \big/ \delta \right)  \right) \ge 1 - \delta.
\end{equation*}
\end{proposition}

\subsection{Proof of Lemma \ref{lem:ls}}
Lemma \ref{lem:ls} shows that the convergence rate of the least squares estimator is dictated by the growth rate of the smallest eigenvalue of the covariates matrx $A_t^\top A_t$. Parts of our proof technique are inspired by recent developments in learning dynamical systems \cite{sarkar2018near}.
\begin{proof} \label{proof:ls}
  Define the event
  $$
  \mathcal{E} = \left\lbrace \exists c > 0,  \exists t_0\ge 0, \forall t \ge t_0, \quad   \frac{1}{t^\alpha} \lambda_{\min}(A_t^\top A_t) > c \right \rbrace.
  $$
 By assumption, $\mathcal{E}$ holds with probability 1. Note that the $t_0, c$ may be random here.
  It also holds on the event $\mathcal{E}$ that for all $t\ge t_0$ we have $2 A_t^\top  A_t \succ A_t^\top A_t + c t^\alpha$ which implies that $2 (A_t^\top A_t + c t^\alpha)^{-1} \succ    (A_t^\top  A_t)^{-1} $.  This means that on the event $\mathcal{E}$, for all $t \ge t_0$, we have $\Vert A_t^\top  E_t\Vert_{(A_s^\top A_s)^{-1}}^2  < 2 \Vert A_t^\top E_t \Vert_{(A_s A_s^\top + c t^\alpha)^{-1}}^2$. Then, using the decomposition \eqref{eq:ls decomposition} we obtain
  \begin{equation}\label{eq:ls upper bound}
  \Vert \hat{\mu}_t - \mu\Vert <  \frac{\sqrt{2} \Vert A_t^\top E_t \Vert_{( A_t^\top  A_t + ct^\alpha)^{-1}}}{\lambda_{\min}\left(A_t^\top  A_t\right)^{1/2}} < \frac{\sqrt{2}}{\sqrt{c}t^{\alpha / 2}}  \Vert A_t^\top E_t \Vert_{( A_t^\top  A_t + ct^\alpha)^{-1}}.
\end{equation}
  We will show that $\Vert A_t^\top E_t \Vert_{( A_t^\top  A_t + ct^\alpha)^{-1}} = o(t ^\beta)$ a.s. for all $\beta > 0$. This will ensure immediately with the upper bound \eqref{eq:ls upper bound} that $ \Vert \hat{\mu}_t - \mu \Vert = o(t^{\beta})$ a.s. for all $\beta \in (0, \alpha/2)$. By Proposition \ref{prop:self concentration}, it holds for all $\beta>0$ and $t \ge 0$
  $$
  \mathbb{P}\left(  \frac{1}{t^{\beta}}  \Vert A_t^\top E_t \Vert_{( A_t^\top  A_t + ct^\alpha)^{-1}} > \frac{\sigma}{t^{\beta}} \left( 2 \log\left( \det\left( (A_t^\top A_t + c t^\alpha I_d )(ct^\alpha I_d)^{-1}\right)^\frac{1}{2} \big/ \delta \right)\right)^{1/2} \right) \le \delta.
  $$
  Since $\mathcal{A}$ is finite, we may upper bound $\det\left((A_t^\top  A_t + c t^\alpha I_d \right)(ct^\alpha I_d)^{-1}) \le (L^2t^{1-\alpha}/c + 1)^d$ where $L = \max_{a \in \mathcal{A}} \Vert a \Vert$ and deduce that
  $$
  \mathbb{P}\left(  \frac{1}{t^{\beta}}  \Vert A_t^\top E_t \Vert_{( A_t^\top  A_t + ct^\alpha)^{-1}} > \frac{\sigma}{t^{\beta}} \left(2 \log\left( L^d t^{\frac{(1-\alpha)d}{2}} \big/ c^{\frac{d}{2}}\delta \right)\right)^{1/2} \right) \le \delta,
  $$
  which we may rewrite after substitution as
  $$
  \mathbb{P}\left(  \frac{1}{t^{\beta}}  \Vert A_t^\top E_t \Vert_{( A_t^\top  A_t + ct^\alpha)^{-1}} > \varepsilon  \right) \le  \frac{L^d}{c^{\frac{d}{2}}} t ^{\frac{(1-\alpha)d}{2}}\exp\left( - \frac{\varepsilon^2 t^{2\beta}}{2 \sigma^2} \right).
  $$
For all $\varepsilon > 0$, since $\sum_{t=1}^\infty t^{\frac{(1-\alpha)d}{2}} \exp(- \frac{\varepsilon^2t^{2\beta}}{2 \sigma^2})< \infty$, we have
  $$
  \sum_{t = 1}^\infty \mathbb{P}\left(  \frac{1}{t^{\beta}}  \Vert A_t^\top E_t \Vert_{( A_t^\top  A_t + ct^\alpha)^{-1}} > \varepsilon  \right) < \infty.
  $$
  Thus, by the first Borell-Cantelli lemma, we have for all $\varepsilon > 0$
  $$
  \mathbb{P}\left(  \left\lbrace \frac{1}{t^\beta}  \Vert A_t^{\top} E_t\Vert_{(A_t^\top A_t + c t^{\alpha})^{-1}} > \varepsilon \right\rbrace  i.o.  \right) = 0.
  $$
 Thus, we have proved that $ \frac{1}{t^{\beta}}  \Vert A_t^\top E_t \Vert_{( A_t^\top  A_t + ct^\alpha)^{-1}} \tends 0$ a.s..
\end{proof}

\subsection{Proof of Lemma \ref{lem:ls concentration}}
The proof of Lemma \ref{lem:ls concentration} is very similar to that of Lemma $\ref{lem:ls}$, but in order to obtain a non-asymptotic concentration bound, a stronger condition is needed, namely a non-asymptotic lower bound for the rate of growth of the smallest eigenvalue of the covariates matrix $A_t^\top A_t$.

\begin{proof} We have by assumption that there are $c > 0$ and $t_0 \ge 0$ such that for all $t\ge t_0$, the event
  $$
  \mathcal{E} = \left\lbrace \lambda_{\min}\left( A_t^\top A_t \right) > ct^\alpha \right\rbrace
  $$
  holds with probability 1. We can now carry the same derivation as in the proof of Lemma \ref{proof:ls} with the distinction that $c, t_0$ are deterministic and conclude that for all $\varepsilon >0$, and $t \ge t_0$, we have
  $$
  \mathbb{P}\left(  \frac{\sqrt{2}}{\sqrt{c}t^{\beta}}  \Vert A_t^\top E_t \Vert_{( A_t^\top  A_t + ct^\alpha)^{-1}} > \varepsilon  \right) \le  \frac{L^d}{c^{\frac{d}{2}}} t ^{\frac{(1-\alpha)d}{2}}\exp\left( - \frac{c \varepsilon^2 t^{2\beta}}{4 \sigma^2} \right),
  $$
  with the choice of $\beta = \alpha/2$ and using the upper bound \eqref{eq:ls upper bound} which can be shown similarly under the event $\mathcal{E}$, we have for all $\varepsilon > 0$, and $t\ge t_0$ that
  $$
  \mathbb{P}\left(  \Vert \hat{\mu}_t - \mu \Vert > \varepsilon  \right) \le  (c^{-1/2}L)^d t ^{\frac{(1-\alpha)d}{2}}\exp\left( - \frac{\varepsilon^2 t^{\alpha}}{2 \sigma^2} \right).
  $$
\end{proof}

\newpage

\section{Stopping rule}

The derivation of our stopping rule is inspired by that of Garivier and Kaufmann \cite{garivier2016optimal} for the MAB setting and relies on the classical generalized log-likelihood ratio (GLLR) test. The main distinction is that in the linear bandit setting, sampling an arm may provide additional statistical information about other arms, therefore one has to consider the full history of observations and sampled arms when comparing arms in the GLLR. We define our GLLR accordingly.

Furthermore, because of the linear structure, we are able to derive an exploration threshold which does not depend on the number of arms $K$, but only on the ambient dimension $d$. Our choice of threshold relies on the deviation bound presented in Proposition \ref{prop:self concentration stopping} (see Lemma 9 in \cite{abbasi2011improved}). But most importantly, to circumvent a naive union bound over the set of arms $\mathcal{A}$, we analyze the stopping time by leveraging the GLLR formulation (see Lemma \ref{lem:GLLR closed form})
 under the event of failure (failure to output the best arm). The stopping rules derived by Soare et al. \cite{soare2014best} follow directly from the deviation bound in \cite{abbasi2011improved}, rather than from the GLLR and consequently, they cannot avoid the dependency on $K$ even for the oracle stopping rule. Most existing algorithms in the literature are phase-based and rely on elimination criteria to stop \cite{fiez2019sequential, tao18best, xu2017}. In these algorithms, the phase transition rules and elimination criteria depend in a way or another on the number of arms $K$.

\subsection{Proof of Lemma \ref{lem:GLLR closed form}}
Here, we show that the generalized log-likelihood ratio can be expressed in a closed form, one that resembles the expression of $\psi$ used in the lower bound.

Let us first recall that, under the gaussian noise assumption, the density function of the sample path $r_1, a_1, \dots, r_t, a_t$ is
$$
f(r_1, a_1, \dots, r_t, a_t) \propto \exp\left( -\frac{1}{2} \sum_{s=1}^t (r_s - \mu^\top a_s)^{2} \right).
$$
Observe that the maximization problem $\max_{\lbrace \mu: \mu^\top (a - b)  \ge  - \varepsilon \rbrace} f_\mu(r_t,a_t, \dots, r_1,a_1)$ is, by monotonicity of the exponential, equivalent to
\begin{align*}
  \min_{\mu} \qquad  &   \frac{1}{2}\sum_{s=1}^t (r_s - \mu^\top a_s)^2  \\
  \textrm{s.t. } \qquad &  \mu^\top (a-b )  \ge -\varepsilon,
\end{align*}
which is a convex program. The optimality conditions give us
\begin{align*}
  \lambda & \ge 0, \\
  \lambda (\varepsilon + \mu^\top (a-b )) & = 0, \\
  -\varepsilon - \mu^\top(a-b)& \le 0, \\
  \left(\sum_{s=1}^t a_s a_s^\top \right) \mu  - \sum_{s=1}^t  a_s r_s + \lambda (a-b) & = 0,
\end{align*}
where $\lambda$ is the Lagrange multiplier associated with the inequality constraint of the problem. Under the assumption that $\sum_{s=1}^t a_s a_s^\top$ is invertible, we introduce the least squares estimator $ \hat{\mu}_t = \left(\sum_{s=1}^t a_s a_s^\top \right)^{-1} \left(\sum_{s=1}^t a_s r_s \right)$. Then from optimality conditions, it follows that
\begin{equation} \label{eq:opt solution}
  \mu_1^* = \begin{cases}
    \hat{\mu}_t & \textrm{if } \quad  \hat{\mu}_t^\top (a-b)  \ge -\varepsilon, \\
    \hat{\mu}_t + (-\varepsilon - \hat{\mu}_t ^\top (a-b))  \frac{\left(\sum_{s=1}^t a_s a_s^\top\right)^{-1} (a-b)}{(a-b)^\top\left(\sum_{s=1}^t a_s a_s^\top\right)^{-1} (a-b)}  & \textrm{otherwise.}
  \end{cases}
\end{equation}
Similarly the solution to the maximization problem $\max_{\lbrace \mu: \langle \mu, a - b \rangle \le  \varepsilon\rbrace} f_\mu(r_t,a_t, \dots, r_1,a_1)$ is
\begin{equation}
  \mu_2^* = \begin{cases}
    \hat{\mu}_t & \textrm{if } \quad \hat{\mu}_t^\top (a-b) \le - \varepsilon, \\
    \hat{\mu}_t + (-\varepsilon - \hat{\mu}_t ^\top (a-b))  \frac{\left(\sum_{s=1}^t a_s a_s^\top\right)^{-1} (a-b)}{(a-b)^\top\left(\sum_{s=1}^t a_s a_s^\top\right)^{-1} (a-b)}  & \textrm{otherwise.}
  \end{cases}
\end{equation}
Hence, the generalized log likelihood ratio can be expressed as
\begin{align*}
  Z_{a,b,\varepsilon}(t) & = \frac{1}{2} (\mu_1^* - \mu_2^*)^\top \left(\sum_{s=1}^t a_s a_s^\top \right) (2\mu_t - \mu^*_1 - \mu^*_2) \\
  & = \textrm{sign}(\mu_t^\top(a-b) + \varepsilon) \frac{( \hat{\mu}_t^\top( a-b) + \varepsilon)^2}{2(a-b)^\top\left(\sum_{s=1}^t a_s a_s^\top\right)^{-1} (a-b)}.
\end{align*} \ep

The following corollary is an immediate consequence of Lemma \ref{lem:GLLR closed form}. Let us recall that $Z_{a,b}(t) = Z_{a,b, 0}(t)$.
\begin{corollary}\label{corr:Z closed form}
Let $t\ge 0$, and assume that $\sum_{s=1}^t a_s a_s^\top \succ 0$. Then for all $\hat{a}_t \in \argmax_{a \in \mathcal{A}} \hat{\mu}_t^\top a$, it holds
  \begin{equation}
    Z(t) = \max_{a \in \mathcal{A}}\min_{b \in \mathcal{A}\backslash \lbrace a\rbrace} Z_{a,b}(t) =  \min_{b \in \mathcal{A}\backslash \lbrace \hat{a}_t \rbrace} Z_{\hat{a}_t,b}(t).
  \end{equation}
\end{corollary}
\begin{proof} Under the assumption that $\sum_{s=1}^t a_s a_s^\top \succ 0$, by Lemma \ref{lem:GLLR closed form}, the sign of $Z_{a,b}(t)$ is that of $\hat{\mu}_t^\top (a-b)$.
Additionally, since $\hat{a}_t \in \argmax_{a \in \mathcal{A}}\hat{\mu}_t^\top a$, it holds for all $b \in \mathcal{A} \backslash \lbrace \hat{a}_t \rbrace$ that  $\hat{\mu}_t^\top(\hat{a}_\tau - b) \ge 0$.
Hence it immediately follows that $Z_{a,b}(t) \ge 0$ if and only if $a \in \argmax_{a \in \mathcal{A}}\hat{\mu}_t^\top a$.
Furthermore, if $\hat{a}_t$ is not unique, then we may find $b \in \argmax \hat{\mu}_t^\top b$ such that $\hat{a}_t \neq b$, and then by Lemma \ref{lem:GLLR closed form} obtain $Z_{\hat{a}_t, b}(t) = 0$.
Hence, we conclude that regardless of whether $\hat{a}_t$ is unique or not, $Z(t) = \min_{b \in \mathcal{A}\backslash \lbrace \hat{a}_t\rbrace}Z_{\hat{a}_t, b}(t)$ .
\end{proof}

\subsection{Proof of Proposition \ref{prop:stopping rule}}

Let us consider the events
\begin{align*}
\mathcal{E}_1 & = \lbrace \tau < \infty \rbrace = \left\lbrace \exists t \in \mathbb{N}^* : \quad  \max_{a \in \mathcal{A}} \min_{b \in \mathcal{A}\backslash \lbrace a\rbrace} Z_{a,b}(t) > \beta(\delta, t) \textrm{ and } \sum_{s=1}^t a_s a_s^\top \succeq c I_d  \right\rbrace,  \\
\mathcal{E}_2 & = \lbrace \mu^\top (a^*_\mu - \hat{a}_\tau) > 0 \rbrace.
\end{align*}
Now note that  if there exists $t \in \mathbb{N}^*$ such that $\sum_{s=1}^t a_s a_s \succeq c I_d$ and $\mu^\top  (a^*_\mu - \hat{a}_t) > 0$ then $\hat{a}_t \neq a^*_\mu$.
Additionally, from Corollary \ref{corr:Z closed form}, we know that under $\mathcal{E}_1$,
that for all $t\ge 1$, it holds that $Z(t) = \min_{b \in \mathcal{A}\backslash \lbrace \hat{a}_t \rbrace} Z_{\hat{a}_t,b}(t)$.
Therefore, we have
\begin{align*}
  \mathcal{E}_1\cap \mathcal{E}_2 & = \left\lbrace \exists t \in \mathbb{N}^*: \;  Z(t) > \beta(\delta, t) \textrm{ and } \sum_{s=1}^t a_s a_s^\top \succeq c I_d \textrm{ and } \mu^\top (a^\star_\mu - \hat{a}_t) > 0 \right\rbrace \\
  & = \left\lbrace \exists t \in \mathbb{N}^*: \;  \min_{b \in \mathcal{A}\backslash \lbrace \hat{a}_t \rbrace} Z_{\hat{a}_t,b}(t) > \beta(\delta, t) \textrm{ and } \sum_{s=1}^t a_s a_s^\top \succeq c I_d \textrm{ and } \mu^\top (a^\star_\mu - \hat{a}_t) > 0 \right\rbrace \\
  &  \subseteq \left\lbrace \exists t \in \mathbb{N}^*:  Z_{\hat{a}_t,a^\star_\mu}(t) > \beta(\delta, t) \textrm{ and } \sum_{s=1}^t a_s a_s^\top \succeq c I_d \textrm{ and } \mu^\top (a^\star_\mu - \hat{a}_t) > 0 \right\rbrace.
\end{align*}
Since under the event $\mathcal{E}_1 \cap \mathcal{E}_2$ and by definition of $\hat{a}_t$, we have $\hat{\mu}_t^\top(\hat{a}_t - a^\star_\mu) \ge 0$, and $\mu^\top(a^\star_\mu - \hat{a}_t) > 0$. In view of \eqref{eq:opt solution}, it follows that
\begin{align*}
  \max_{\lbrace \mu': (\mu')^\top (\hat{a}_t - a^\star_\mu) \ge  0 \rbrace } f_{\mu'} (r_t,a_t, \dots, r_1,a_1) & = f_{\hat{\mu}_t} (r_t,a_t, \dots, r_1,a_1), \\
    \max_{\lbrace \mu': (\mu')^\top (\hat{a}_t - a^\star_\mu) \le  0 \rbrace } f_{\mu'} (r_t,a_t, \dots, r_1,a_1) & \ge f_{{\mu}} (r_t,a_t, \dots, r_1,a_1).
\end{align*}
Thus under $\mathcal{E}_1 \cap \mathcal{E}_2$ it holds that
\begin{align*}
  Z_{\hat{a}_t, a^*_\mu}(t) & = \log\left(  \frac{\max_{ \mu': (\mu')^\top (\hat{a}_t - a^\star_\mu) \ge  0} f_{\mu'} (r_t,a_t, \dots, r_1,a_1)}{\max_{ \mu': (\mu')^\top (\hat{a}_t - a^\star_\mu) \le  0} f_{\mu'} (r_t,a_t, \dots, r_1,a_1)} \right) \\
  & \le \log\left(  \frac{f_{\hat{\mu}_t}(r_t,a_t, \dots, r_1,a_1)}{f_\mu(r_t,a_t, \dots, r_1,a_1)} \right) &  \\
  & = \frac{1}{2} (\hat\mu_t - \mu)^\top \left(\sum_{s=1}^t a_s a_s^\top \right) (\hat\mu_t  - \mu) \\
  & = \frac{1}{2}\Vert \mu - \hat\mu_t \Vert^{2}_{\sum_{s=1}^t a_s a_s^\top},
\end{align*}
which further implies that
\begin{align*}
  \mathcal{E}_1\cap \mathcal{E}_2 & \subseteq \bigg\{ \exists t \in \mathbb{N}^*:  \quad  \frac{1}{2}\Vert \mu - \hat\mu_t \Vert^{2}_{\sum_{s=1}^t a_s a_s^\top}  \ge \beta(\delta, t) \textrm{ and } \sum_{s=1}^t a_s a_s^\top \succeq c I_d \\ 
  &\qquad \textrm{ and } \mu^\top (a^*_\mu - \hat{a}_t) > 0 \bigg\}\\
   &  \subseteq \left\lbrace \exists t \in \mathbb{N}^* :  \quad  \frac{1}{2}\Vert \mu - \mu_t \Vert^{2}_{\sum_{s=1}^t a_s a_s^\top}  \ge \beta(\delta, t) \textrm{ and } \sum_{s=1}^t a_s a_s^\top \succeq c I_d\right\rbrace.
\end{align*}
We note that when $\sum_{s =1}^t a_s a_s^\top \succeq c I_d$, then for all $\rho > 0$, $(1+\rho) \sum_{s=1}^t a_s a_s^\top \succeq \sum_{s=1}^t a_s a_s^\top  + \rho c I_d$, which means that $ (1+\rho) (\sum_{s=1}^t a_s a_s^\top +\rho c I_d)^{-1}\succeq (\sum_{s=1}^t a_s a_s^\top)^{-1}$. Thus, we may have
$$
\Vert \hat\mu_t - \mu\Vert^2 = \Big \Vert \sum_{s=1}^t a_s \eta_s \Big\Vert_{\left( \sum_{s = 1}^t a_s a_s^\top \right)^{-1}}^2 \le (1+\rho) \Big \Vert \sum_{s=1}^t a_s \eta_s \Big\Vert_{\left( \sum_{s = 1}^t a_s a_s^\top  + \rho c I_d \right)^{-1}}^2.
$$
This leads to
\begin{align*}
    \mathcal{E}_1\cap \mathcal{E}_2 & \subseteq  \left\lbrace \exists t \in \mathbb{N}^* :  \quad  \frac{1}{2}(1+\rho) \Big \Vert \sum_{s=1}^t a_s \eta_s \Big\Vert_{\left( \sum_{s = 1}^t a_s a_s^\top  + \rho c I_d \right)^{-1}}^2  \ge \beta(\delta, t) \right\rbrace,
\end{align*}
and with the choice
\begin{equation*}
\beta(\delta, t) = (1+\rho) \sigma^2 \log\left( \frac{\det( (\rho c)^{-1}\sum_{s=1}^t a_s a_s^\top + I_d)^{1/2}}{\delta}\right),
\end{equation*}
we write
\begin{align*}
     & \mathcal{E}_1\cap \mathcal{E}_2 \subseteq \\
   &   \left\lbrace \exists t \in \mathbb{N}^* :  \frac{1}{2}\Big \Vert \sum_{s=1}^t a_s \eta_s \Big\Vert_{\left( \sum_{s = 1}^t a_s a_s^\top  + \rho c I_d \right)^{-1}}^2  > \sigma^2 \log\left( \frac{\det( (\rho c)^{-1}\sum_{s=1}^t a_s a_s^\top + I_d)^{1/2}}{\delta}\right) \right\rbrace.
\end{align*}
Finally, it follows immediately from Proposition \ref{prop:self concentration stopping} that
$$
\mathbb{P}\left(\tau < \infty, \mu^\top(a^\star_\mu - \hat{a}_\tau) > 0\right) = \mathbb{P}(\mathcal{E}_1\cap \mathcal{E}_2) \le \delta.
$$
\ep

Proposition \ref{prop:stopping rule} does not yet guarantee that we have a $\delta$-PAC startegy. However, a sufficient condition for any strategy with the proposed decision rule and stopping rule to be $\delta$-PAC, is to simply ensure that $\mathbb{P}(\tau < \infty) = 1$. This condition will have to be satisfied by our sampling rule.
\begin{corollary}[$\delta$-PAC guarantee]
For any strategy using the proposed decision rule and stopping rule and such that $\mathbb{P}(\tau < \infty)$, it is guaranteed that $\mathbb{P}(\mu^\top (a^\star_\mu - \hat{a}_t) > 0) \le \delta$.
\end{corollary}

\newpage

\section{Sampling rule}

Our sampling rule as described in Section 3.5 is based on tracking a sequence of allocations that provably approaches the set of optimal allocations. This set of optimal allocations $C^\star(\mu)$ that is not necessarily a singleton as in the multi-armed bandit setting \cite{kaufmann2016complexity}. This makes the analysis extremely challenging. However by crucially leveraging the geometric properties of this set and the continuity properties of $\psi$ and $C^\star(\mu)$ we are able to prove that tracking is possible.

Additionally, we choose arms from the support (set of non zero elements) of the average allocations up to the current round. This is motivated by the fact that when $K$ is exceedingly large in comparison with the dimension $d$, it is possible to represent any matrix $A$ in the convex hull $\textrm{conv}(\lbrace a a^\top: a \in \mathcal{A}\rbrace)$ by an allocation $w$ with support of at most $O(d^2)$ such that $A = \sum_{a \in \mathcal{A}} w_a a a^\top$. This observation was made by Soare et el. \cite{soare2014best} and follows from Caratheorody's Theorem. A consequence of this sampling strategy is reflected in Lemma \ref{lem:tracking lemma}.

One further novel part of the analysis is the introduction of \textit{laziness}, the idea that the algorithm does not need to perform a computationally demanding task at every round. In the linear bandit setting this computationally demanding task is the optimization problem $\max_{w \in \Lambda} \psi(\hat{\mu}_t, w)$. Existing algorithms in the literature resort to phase-based schemes such us gap elimination in order to attain efficiency. However these schemes often fail to fully stitch the statistical information between phases. This can be seen in the least squares constructions of the algorithms $\mathcal{XY}$-adaptive \cite{soare2014best}, ALBA \cite{tao18best}, RAGE \cite{fiez2019sequential} where the samples from previous phases are discarded. Our tracking rule allows for a natural flow of information between rounds regardless of the laziness of the algorithm. This is shown by Proposition \ref{prop:tracking}.

We shall now prove Proposition \ref{prop:tracking} and all the related lemmas. Lemma \ref{lem:sufficient exploration} shows that we have sufficient exploration. Lemma \ref{lem:tracking lemma} is the crucial step in our analysis here. It's a tracking lemma that formalizes the idea that we may track a sequence that converges to a set $C$ rather than a point. The proof requires the convexity of the set $C$. In the main analysis of the sampling rule $C$ is replaced by $C^\star(\mu)$.

\subsection{Proof of Lemma \ref{lem:sufficient exploration}}
The idea of the proof is to show that if at some time $t_0 + 1$, the condition $\lambda_{\min}(\sum_{s=1}^t a_s a_s^\top) > f(t)$ is violated, then the number of rounds needed to satisfy the condition again cannot exceed $d$ rounds.

First, we note that  $d = \inf\lbrace t\ge 1: \lambda_{\min}(\sum_{s=1}^t a_s a_s^\top) \ge f(t)\rbrace$. Indeed, we have by construction that for all $t < d$, $\lambda_{\min}(\sum_{s=1}^t a_s a_s^\top) = 0$ and $\lambda_{\min}\left(\sum_{s=1}^d a_s a_s^\top\right) = \lambda_{\min}\left(\sum_{a \in \mathcal{A}_0} a a^\top \right) = f(d)$.
  \medskip
  Now, if there exists $t_0 \ge d$, such that $\lambda_{\min}\left(\sum_{s=1}^{t_0} a_s a_s^\top \right) \ge f(t_0)$ and $\lambda_{\min} \left(\sum_{s=1}^{t_0+1} a_s a_s^\top \right) < f(t_0+1)$, then we may define $t_1 = \inf \left\lbrace t > t_0: \lambda_{\min}\left( \sum_{s=1}^{t} a_s a_s^\top \right) \ge f(t) \right\rbrace$. Let us observe that for all $t_0 \le t \le t_1$, we have
  $$
  \lambda_{\min}\left( \sum_{s=1}^t a_s a_s^\top \right) \ge \lambda_{\min}\left( \sum_{s=1}^{t_0} a_s a_s^\top \right) \ge f(t_0).
  $$
  Note that if $t_1 \ge t_0 + d + 1$, then, by construction, we have
  \begin{align*}
  \lambda_{\min} \left( \sum_{s=1}^{t_1} a_s a_s^\top \right) & \ge \lambda_{\min}\left(  \sum_{s=1}^{t_0+ d +1} a_s a_s^\top \right) \\
   & \ge \lambda_{\min}\left(  \sum_{s=1}^{t_0+1} a_s a_s^\top + \sum_{a \in \mathcal{A}_0} a a^\top \right) \\
   & \ge \lambda_{\min}\left(  \sum_{s=1}^{t_0+1} a_s a_s^\top \right)  + \lambda_{\min}\left( \sum_{a \in \mathcal{A}_0} a a^\top \right) \\
   &  = \lambda_{\min}\left(\sum_{s=1}^{t_0+1} a_s a_s^\top\right) + c_{\mathcal{A}_0}\sqrt{d} \\
   & \ge \lambda_{\min}\left(\sum_{s=1}^{t_0} a_s a_s^\top\right) + c_{\mathcal{A}_0}\sqrt{d} \\
   & \ge f(t_0) + c_{\mathcal{A}_0}\sqrt{d}.
\end{align*}
  However, we have
  $$
  t_0 \ge \frac{1}{4}\left(d + \frac{1}{d} + 2\right) \implies   \sqrt{t_0 + d + 1} + \sqrt{t_0} \ge \sqrt{d}+\frac{1}{\sqrt{d}} \implies f(t_0) + c_{\mathcal{A}_0} \sqrt{d} \ge f(t_0 + d + 1).
  $$
  Therefore, if $t_0 \ge \frac{1}{4}\left(d + \frac{1}{d} + 2\right)$, then it holds that $t_1 \le t_0+d +1$. In other words, we have shown that for all $t \ge \frac{1}{4}\left(d + \frac{1}{d} + 2\right) + d + 1$, we have
  $$
  \lambda_{\min}\left(\sum_{s=1}^t a_s a_s^\top\right) \ge f(t-d-1).
  $$
\ep

\subsection{Proof of Lemma \ref{lem:tracking lemma}}
Our proof for the tracking lemma is inspired by that of D-tracking for linear bandits by Garivier and Kaufmann \cite{garivier2016optimal}. We follow similar steps but there are crucial differences. The main one lies in the fact that we have a sequence that converges to a set $C$ rather than to a unique point. The convexity of $C$ is a crucial point in our analysis as it allows to show that tracking the average of this converging sequence will eventually allow our empirical allocation to be sufficiently close to the set $C$. Intuitively, the average is a stable point to track. Furthermore, we also highlight the fact that the sparsity of the average allocations $\sum_{s=1}^t w(s)/t$ is reflected in the error by which $(N_a(t))_{a \in \mathcal{A}}$ approaches the set $C$. This is due to the nature of our sampling rule as shall be proven.

\begin{proof}
For all $t \ge 1$ denote
$$
\overline{w}(t) = \frac{1}{t}\sum_{s=1}^t w(s).
$$
Since $C$ is non-empty and compact, we may define
$$
\hat{w}(t) = \argmin_{w \in C} d_\infty( \overline{w}(t), w).
$$
Note that by convexity of $C$, there exists $t_0' \ge t_0$ such that $\forall t \ge t_0'$, \\
$d_\infty((N_{a}(t) / t)_{a \in \mathcal{A}}, C) \le d_\infty((N_{a}(t) / t)_{a \in \mathcal{A}}, \hat{w}(t))$ and $d_{\infty}(\overline{w}(t), \hat{w}(t)) \le 2 \varepsilon$.

To see that, let us define for all $t \ge 1$, $v(t) = \argmin_{w \in C} d_\infty(w, w(t))$, and observe that for all $ a \in \mathcal{A}$, we have
\begin{align*}
 \left\vert \frac{1}{t}\sum_{s=1}^t w_a(s) - \frac{1}{t}\sum_{s=1}^t v_a(s)\right\vert  & \le \frac{1}{t} \sum_{s=1}^{t_0} \left\vert w_a(s) - v_a(s)\right\vert + \frac{1}{t} \sum_{s=t_0 + 1}^t \left\vert w_a(s) - v_a(s)\right\vert \\
 & \le \frac{t_0}{t} + \frac{t-t_0}{t} \varepsilon.
\end{align*}
Thus if $t\ge t_0' = \frac{t_0}{\varepsilon}$, then $d_\infty(\overline{w}(t), \frac{1}{t}v(t)) \le 2 \varepsilon$. Finally since $\frac{1}{t} \sum_{s=1}^t v(s) \in C$ (by convexity of $C$), it follows that
$$
\forall t \ge t_0'  \qquad d_\infty(\overline{w}(t), \hat{w}(t)) \le d_\infty\left( \overline{w}(t), \frac{1}{t} \sum_{s=1}^t v(s)\right) \le 2\varepsilon.
$$
We further define for all $t\ge  1$, $\varepsilon_{a, t} = N_a(t) - t \hat{w}_a(t)$. The main step of the proof is to show that there exists $t_0'' \ge t_0'$ such that for all $t\ge t_0''$, for all $a \in \mathcal{A}$ we have
$$
\lbrace a_{t+1} = a \rbrace \subseteq  \mathcal{E}_1(t) \cup \mathcal{E}_2(t) \subseteq \lbrace \varepsilon_{a, t} \le 6 t \varepsilon\rbrace,
$$
where
\begin{align*}
  \mathcal{E}_1(t) & = \left \lbrace a = \argmin_{a \in \textrm{supp}(w_t)} (N_a(t) - t \overline{w}_a(t)) \right\rbrace,\\
  \mathcal{E}_2(t) & = \left\lbrace  \lambda_{\min}\left(\sum_{s=1}^t a_s a_s^\top \right) < f(t) \quad and \quad  a = \mathcal{A}_0(i_t)  \right\rbrace.
\end{align*}

The first inclusion is immediate by construction. Now let $t \ge t_0$, we have:
\begin{itemize}
  \item[\textit{(Case 1)}] If $\lbrace a_{t+1} = a \rbrace \subseteq \mathcal{E}_1(t)$, then we have
  \begin{align*}
    \varepsilon_{a,t} & = N_a(t) - t \hat{w}_a(t) \\
    & = N_a(t) -  t \overline{w}_a(t) + t\overline{w}_a(t) -t \hat{w}_a(t) \\
    & \le N_a(t) - t \overline{w}_a(t) + t \varepsilon & (\textrm{since } d_\infty(\hat{w}(t), \overline{w}(t))\le \varepsilon)\\
    & \le \min_{a \in \textrm{supp}(\overline{w}(t))} N_a(t) - t\overline{w}_a(t) + t \varepsilon & (\textrm{since } \mathcal{E}_1(t) \textrm{ holds}) \\
    & \le 2t \varepsilon,
  \end{align*}
  where the last inequality holds because
  $$
    \sum_{a \in \textrm{supp}(\overline{w}(t))} N_a(t) - t \overline{w}_a(t) = - \sum_{a  \in\mathcal{A} \backslash \textrm{supp}(\overline{w}(t))} N_a(t) \le 0
  $$
  thus $\mathcal{E}_2(t) \subseteq \lbrace \varepsilon_{a,t} \le 2t \varepsilon\rbrace$.
  \item[\textit{(Case 2)}] If $\lbrace a_{t+1} = a \rbrace \subseteq \mathcal{E}_2(t)$, then it must hold that $a \in \mathcal{A}_0$. Let us define for al $k\ge 1$
   \begin{align*}
     N_{a,1}(k) & = \sum_{s=1}^k 1_{\left\lbrace a_{k} = a \textrm{ and }  \lambda_{\min} \left( \sum_{s=1}^{k-1} a_s a_s^\top \right) < f(k-1) \right\rbrace}, \\
     N_{a,2}(k) & = \sum_{s=1}^k 1_{\left\lbrace a_{k} = a \textrm{ and }   \lambda_{\min} \left( \sum_{s=1}^{k-1} a_s a_s^\top \right) \ge f(k-1) \right\rbrace}.
   \end{align*}
 Note that $N_a(k) = N_{a,1}(k) + N_{a,2}(k)$ and that $ N_{a,1}(k) - 1 \le \min_{a \in \mathcal{A}_0} N_{a,1}(k) \le N_{a,1}(k)$. The latter property follows from the forced exploration sampling scheme. Now, since the event $\mathcal{E}_2(t)$ holds, we observe that
   \begin{align*}
     (N_{a,1}(t) - 1) \le \min_{a \in \mathcal{A}_0} N_{a,1}(t) \lambda_{\min}\left(\sum_{a \in \mathcal{A}_0} a a^\top\right) \le  \lambda_{\min}\left(\sum_{s=1}^t a_s a_s^\top \right) < f(t)
   \end{align*}
   and since $f(t) = \lambda_{\min}\left(\sum_{a \in \mathcal{A}_0} a a^\top\right) \frac{\sqrt{t}}{\sqrt{d}}$, we obtain
   $$
   N_{a,1}(t) \le \sqrt{t}/\sqrt{d} + 1.
   $$
   Next, let $k \le t$ be the largest integer such that $N_{a,2}(k) = N_{a,2}(k-1) + 1$. Note that at such $k$ the event $\mathcal{E}_1(k-1)$ must hold by definition of $N_{a,2}(k-1)$, and we have
   $$
   N_{a,2}(t) = N_{a,2}(k) = N_{a, 2}(k-1) + 1 \quad \textrm{and} \quad a_{k} = \argmin_{a \in \textrm{supp}(w(k-1))} N_a(k-1) - k \overline{w}_a(k-1).
   $$
 Now we write
   \begin{align*}
     \varepsilon_{a,t} & = N_{a,t} - t\hat{w}_a(t) \\
     & = N_{a,1}(t) + N_{a,2}(t) - t \hat{w}_a(t) \\
     & \le \sqrt{t}/\sqrt{d} + 1  + N_{a,2}(t) - t \hat{w}_a(t).
   \end{align*}
 If $k-1 \le t_0'$, then we have $N_{a,2}(k) \le t_0'$, otherwise since $\mathcal{E}_1(k-1)$ holds, we have
   \begin{align*}
   N_{a, 2}(t) & = 1 + N_{a,2}(k-1) - (k-1) \hat{w}_a(k-1) + (k-1) \hat{w}_a(k-1) \\
   &  \le 1+ 2(k-1)\varepsilon + (k-1) \hat{w}_a(k-1).
  \end{align*}
  Thus
  \begin{align*}
    \varepsilon_{a,t} & \le \sqrt{t}/\sqrt{d} + 1 + \max\lbrace t_0', 1 + 2(k-1)\varepsilon + (k-1) \hat{w}_a(k-1) - t \hat{w}_a(t) \rbrace,
  \end{align*}
  and since
  \begin{align*}
    (k-1)\hat{w}_{a}(k-1) - t  & \hat{w}_a(t)  = (k-1)\hat{w}_{a}(k-1) - (k-1)\overline{w}_{a}(k-1) \\
    &\qquad \qquad + (k-1)\overline{w}_{a}(k-1) - t \hat{w}_a(t) \\
    & \le (k-1)\hat{w}_{a}(k-1) - (k-1)\overline{w}_{a}(k-1)+ t\overline{w}_{a}(t) - t \hat{w}_a(t) \\
    & \le 2(k-1)\varepsilon + 2t \varepsilon\\
    & \le 4 t \varepsilon,
  \end{align*}
  it follows that
  $$
  \varepsilon_{a,t} \le \sqrt{t}/\sqrt{d} + 1 + \max\lbrace t_0',  1+ 6t\varepsilon \rbrace.
  $$
 We conclude that for $t \ge t_0'' = \max \left\lbrace \frac{1}{\varepsilon}, \frac{1}{\varepsilon^2d}, \frac{t_0'}{\varepsilon}\right\rbrace$, it holds that
  $$
  \varepsilon_{a,t} \le 9t\varepsilon
  $$
  and consequently that $\mathcal{E}_2(t) \subseteq \lbrace  \varepsilon_{a,t} \le 9 t\varepsilon\rbrace$. So we have shown that for all $t\ge t_0''$, for all $a \in \mathcal{A}$, it holds that
  $$
  \lbrace a_{t+1} = a \rbrace \subseteq \lbrace \varepsilon_{a,t} \le 9t\varepsilon \rbrace.
  $$
\end{itemize}
The remaining part of the proof is very similar to that of Lemma 17 in \cite{garivier2016optimal}. It can be immediately shown that for $t \ge t_0''$, one has
$$
\varepsilon_{a,t} \le \max(\varepsilon_{a,t_0''}, 9t\varepsilon + 1) \le \max(t_0'', 9t\varepsilon + 1)
$$
Furthermore, note that for all $t \ge 1$ we have $\textrm{supp}(\overline{w}(t)) \subseteq  \textrm{supp}(\overline{w}(t+1))$ since for all $a\in \mathcal{A}$, we have $t\overline{w}_a(t) \le (t+1)\overline{w}_a(t+1)$. Therefore
$$
\sum_{a \in \textrm{supp}(\overline{w}(t))\cup \mathcal{A}_0} \varepsilon_{a,t} =  \sum_{a \in \mathcal{A}\backslash \textrm{supp}(\overline{w}(t))\cup \mathcal{A}_0}  t \hat{w}_a(t) \ge 0.
$$
Thus denoting $p_t = \vert \textrm{supp}(\overline{w}(t)) \vert \backslash \mathcal{A}_0$, we have
\begin{align*}
&\forall a \in \textrm{supp}(\overline{w}(t))\cup\mathcal{A}_0,  \quad  \max(t_0'', 9t\varepsilon + 1)\ge  \varepsilon_{a,t} \ge - (p_t + d - 1)  \max(t_0'', 9t\varepsilon + 1), \\
&\forall a \in \mathcal{A}\backslash \textrm{supp}(\overline{w}(t))\cup\mathcal{A}_0,  \quad  \phantom{\max(t_0'', 9t\varepsilon +} 0 \ge  \varepsilon_{a,t} \ge -t\varepsilon,
\end{align*}
which implies that for all $t\ge t_0''$
$$
\max_{a \in \mathcal{A}} \vert \varepsilon_{a,t}\vert \le (p_t + d - 1) \max(t_0'', 9t \varepsilon + 1)\le (p_t + d - 1) \max(t_0'', 10).
$$
This finally implies that for $t_1 = \frac{1}{\varepsilon}\max\lbrace t_0'', 10\rbrace$, we have for all $t \ge t_1$,
$$
d_\infty(x(t), C^* )\le d_\infty((N_a(t)/t)_{a \in \mathcal{A}}, \hat{w}(t))= \max_{a\in \mathcal{A}} \vert N_a(t)/t  - \hat{w}_a(t)\vert= \max_{a\in \mathcal{A}}\left\vert \frac{\varepsilon_{a,t}}{t} \right\vert \le (p_t + d -1) \varepsilon.
$$
More precisely, we have
$$
t_1(\varepsilon) = \max\left\lbrace \frac{1}{\varepsilon^2}, \frac{1}{\varepsilon^3 d}, \frac{t_0(\varepsilon)}{\varepsilon^3}, \frac{10}{\varepsilon}\right\rbrace.
$$
\end{proof}

\subsection{Proof of Proposition \ref{prop:tracking}}

Let $\varepsilon> 0$. First, by Lemma \ref{lem:continuity}, there exists $\xi(\varepsilon) > 0$ such that for all $\mu'$ such that $\Vert \mu - \mu'\Vert< \xi(\varepsilon)$, it holds that $\max_{w \in C^\star(\mu')} d_\infty(w, C^\star(\mu)) < \varepsilon/2$.

  By Lemma \ref{lem:sufficient exploration}, we have a sufficient exploration. That is $\liminf_{t\to \infty} t^{-1/2}\lambda_{\min}(\sum_{s=1}^t a_s a_s^\top)  > 0$. Thus, by Lemma \ref{lem:ls},  $\hat{\mu}_t$ converges almost surely to $\mu$ with a rate of order $o(t^{1/4})$. Consequently, there exists $t_0 \ge 0$ such that for all $t\ge t_0$, we have $\Vert\mu - \hat{\mu}_t \Vert \le \xi(\varepsilon)$.

  The lazy condition \eqref{eq:lazycond} states that there exists a sequence $(\ell(t))_{t\ge 1}$ of integers such that $\ell(1) = 1$, $\ell(t) \le t$ and $\lim_{t\to \infty}\ell(t) = \infty$, and  $\lim_{t \to \infty}\inf_{s \ge \ell(t)} d_\infty(w(t), C^\star(\hat{\mu}_s)) = 0$ a.s. This guarantees that there exists $t_1 \ge 1$, there exists a sequence $(h(t))_{t\ge 1}$ of integers such that for all $t\ge t_1$, we have $h(t) \ge \ell(t)\ge t$ and $ d_\infty(w(t), C^\star(\hat{\mu}_{h(t)})) < \varepsilon/2$.
  Now for all $t \ge t_0 \vee t_1$,  we have
  \begin{equation*}
    d_{\infty}(w(t), C^\star(\mu)) \le d_\infty(w(t), C^\star(\hat{\mu}_{h(t)}) ) + \max_{w \in C^\star(\hat{\mu}_{h(t)})}d_\infty(w, C^\star(\mu)) < \varepsilon.
  \end{equation*}
  We have shown that $ d_\infty(w(t), C^\star(\mu)) \tends 0$ a.s. Next, we recall that by Lemma \ref{lem:continuity}, $C^\star(\mu)$ is non empty, compact and convex. Thus, applying Lemma \ref{lem:tracking lemma} yields immediately that $d_\infty((N_a(t)/t)_{a \in \mathcal{A}}, C^\star(\mu)) \tends 0$ a.s.. \ep

\newpage
\section{Sample complexity}

We will use the following technical lemma which can be found for instance in \cite{garivier2016optimal}.
\begin{lemma}[Lemma 18 \cite{garivier2016optimal}]\label{lem:technical}
  For any two constants $c_1, c_2 > 0$, and $c_2/c_1 > 1$ we have
  \begin{equation}
    \inf \left\lbrace t \in \mathbb{N}^*: \; c_1 t \ge \log(c_2 t)  \right\rbrace  \le \frac{1}{c_1} \left ( \log\left(\frac{c_2e}{c_1}\right) + \log\log\left(\frac{c_2}{c_1} \right)  \right)
  \end{equation}
\end{lemma}

\subsection{Proof of Theorem \ref{th:sample1}}
The proof of the almost sure sample complexity result follows naturally from the continuity of $\psi$ (see Lemma \ref{lem:continuity0}) and of $C^\star(\mu)$ (see Lemma \ref{lem:continuity}).

We start by defining the event
  $$
  \mathcal{E} = \left\lbrace d_\infty(\left(N_a(t)/t \right)_{a \in \mathcal{A}}, C^\star(\mu) ) \underset{t \to \infty}{\longrightarrow} 0 \textrm{ and } \hat{\mu}_t  \underset{t \to \infty}{\longrightarrow} \mu \right\rbrace.
  $$
  Observe that $\mathcal{E}$ holds with probability 1. This follows from  Lemma \ref{lem:ls}, Lemma \ref{lem:sufficient exploration} and Proposition \ref{prop:tracking}. Let $\varepsilon>0$. By continuity of $\psi$, there exists an open neighborhood $\mathcal{V}(\varepsilon)$ of $\lbrace \mu\rbrace\times C^\star(\mu)$ such that for all $(\mu',w') \in \mathcal{V}(\varepsilon)$, it holds that
  $$
  \psi(\mu', w') \ge (1-\varepsilon) \psi(\mu, w^\star),
  $$
  where $w^\star$ is some element in $C^\star(\mu)$. Now, observe that under the event $\mathcal{E}$, there exists $t_0 \ge 1$ such that for all $t\ge t_0$ it holds that $(\hat{\mu}_t, \left(N_a(t)/t \right)_{a \in \mathcal{A}}) \in \mathcal{V}(\varepsilon)$, thus for all $t \ge t_0$, it follows that
  $$
  \psi(\hat{\mu}_t, \left(N_a(t)/t \right)_{a \in \mathcal{A}}) \ge (1- \varepsilon) \psi(\mu, w^*).
  $$
  Since $\hat{\mu}_t \tends \mu$ and $a^\star_\mu$ is unique, there exists $t_1 \ge 0$ such that for all $t \ge t_1$, $\hat{a}_t$ is unique. Thus, by Lemma \ref{corr:Z closed form}, we may write
  $$
  Z(t) = \min_{a \neq a_{\hat{\mu}_t}^*} \frac{\hat{\mu}_t^\top (a^*_{\hat{\mu}_t} - a )^2}{ 2(a_{\hat{\mu}_t}^* - a)^\top \left(\sum_{s=1}^t a_s a_s^\top\right)^{-1}  (a_{\hat{\mu}_t} - a) } = t \psi(\hat{\mu}_t,  \left(N_a(t)/t \right)_{a \in \mathcal{A}}  ).
  $$
  By Lemma \ref{lem:sufficient exploration}, there exists $t_2 \ge 1$ such that for all $t \ge t_2$ we have
  $$
  \sum_{s=1}^t a_s a_s^\top \succ c I_d.
  $$
  Hence, under the event $\mathcal{E}$, for all $t \ge \max \lbrace t_0, t_1, t_2 \rbrace$, 
  $$
  Z(t) \ge t (1-\varepsilon) \psi(\mu, w^\star) \textrm{ and } \sum_{s=1}^t a_s a_s^\top \succ c I_d.
  $$
  This implies that
  \begin{align*}
    \tau_\delta & = \inf\left \lbrace t \in \mathbb{N}^*: Z(t) > \beta(\delta, t) \quad \textrm{ and } \quad \sum_{s=1}^t a_s a_s^\top \succeq c I_d  \right\rbrace \\
    & \le \max\lbrace t_0, t_1, t_2\rbrace \vee \inf \lbrace t \in \mathbb{N}^*: (1 - \varepsilon) t \psi(\mu, w^\star) > \beta(\delta, t) \rbrace \\
    & \le \max\lbrace t_0, t_1, t_2\rbrace \vee \inf \left\lbrace t \in \mathbb{N}^*: (1 - \varepsilon) t \psi(\mu, w^\star) > c_1 \log\left(\frac{c_2 t^\gamma}{\delta}\right) \right\rbrace \\
    & \lesssim \max \left\lbrace t_0, t_1, t_2 , \frac{1}{1-\varepsilon}T^*_\mu \log\left(\frac{1}{\delta}\right)\right\rbrace,
  \end{align*}
  where $c_1, c_2, \gamma$ denote the positive constants independent of $\delta$ and $t$ that appear in the definition of $\beta(t, \delta)$ (see \eqref{eq:rate}). We used Lemma \ref{lem:technical} in the last inequality for $\delta$ sufficiently small. This shows that $\mathbb{P}(\tau_\delta < \infty) = 1$ and in particular that
  $$
  \mathbb{P}\left( \limsup_{\delta \to 0} \frac{\tau_\delta}{\log\left(\frac{1}{\delta}\right)}  \lesssim T^*_\mu \right) = 1.
  $$
\ep

\subsection{Proof of Theorem \ref{th:sample2}}


Compared to the almost sure result, the expected sample complexity guarantee is more difficult to prove. We break our analysis into three steps. In the first step, we construct a sequence of events over which the stopping time that defines our stopping rule is well-behaved. This requires precise manipulations of the continuity properties of $\psi$ and $C^\star(\mu)$ in combination with the tracking Lemma \ref{lem:tracking lemma}. In the second step, we show indeed that on these events, the stopping time is upper bounded up to a constant by the optimal sample complexity. In the third step, we show that the probabilities of the events under which the sample complexity is not well-behaved are negligible. This is guaranteed thanks to the lazy condition \eqref{eq:lazycond2} and the sufficient exploration (ensured by Lemma \ref{lem:sufficient exploration} under our sampling rule). We finally conclude by giving the upper bound on the expected sample complexity.

\begin{proof}

  Let $\varepsilon > 0$.

  \paragraph{Step 1.} By continuity of $\psi$ (see Lemma \ref{lem:continuity0}), there exists $\xi_1(\varepsilon)>0$ such that for all $\mu'\in \mathbb{R}^d$ and $w' \in \Lambda$
  \begin{equation}\label{eq:construction xi1}
    \begin{cases}
      \Vert \mu' - \mu \Vert & \le \xi_1(\varepsilon) \\
      d_\infty(w', C^\star(\mu)) & \le \xi_1(\varepsilon)
    \end{cases} \implies \vert \psi(\mu, w^\star) - \psi(\mu', w') \vert \le  \varepsilon \psi(\mu, w^\star) = \varepsilon (T^\star_\mu)^{-1}
  \end{equation}
  for any $w^\star \in \argmin_{w \in C^\star(\mu)} d_\infty(w', w)$ (we have $w^\star\in C^\star(\mu)$). Furthermore, by the continuity properties of the correspondance $C^\star$ (see Lemma \ref{lem:continuity}), there exists $\xi_2(\varepsilon) > 0 $ such that for all $\mu' \in \mathbb{R}^d$
  $$
  \Vert \mu - \mu' \Vert \le \xi_2(\varepsilon) \implies   \max_{w''\in C^\star(\mu')} d_\infty(w'', C^\star(\mu)) < \frac{\xi_1(\varepsilon)}{2(K-1)}
  $$
  Let  $\xi(\varepsilon) = \min(\xi_1(\varepsilon), \xi_2(\varepsilon))$. In the following, we construct $T_0$, and for each $T \ge T_0$ an event $\mathcal{E}_T$, under which for all $t \ge T$, it holds
  \begin{equation*}
      \Vert \mu - \hat{\mu}_t \Vert \le \xi(\varepsilon) \implies d_\infty((N_a(t)/t)_{a\in\mathcal{A}}, C^\star(\mu)) \le \xi_1(\varepsilon)
  \end{equation*}
  Let $T \ge 1$, and define the following events
  \begin{align*}
      \mathcal{E}_{1,T} & = \bigcap_{t = \ell(T)}^\infty  \left\lbrace  \Vert \mu - \hat{\mu}_t\Vert \le  \xi(\varepsilon) \right\rbrace \\
      \mathcal{E}_{2,T} & = \bigcap_{t = T }^\infty \left\lbrace \inf_{s\ge \ell(t)} d_\infty(w(t), C^\star(\hat{\mu}_s)) \le \frac{\xi_1(\varepsilon)}{4(K-1)}  \right\rbrace \\
      & \subseteq \bigcap_{t = T }^\infty \left\lbrace \exists s \ge \ell(t): \; d_\infty(w(t), C^\star(\hat{\mu}_s)) \le \frac{\xi_1(\varepsilon)}{2(K-1)} \right\rbrace.
  \end{align*}
  Note that, under the event $\mathcal{E}_{1,T} \cap \mathcal{E}_{2,T}$, we have for all $t \ge T$, there exists $s \ge \ell(t)$ such that
  \begin{align*}
  d_\infty(w(t), C^\star(\mu)) &\le d_\infty(w(t), C^\star(\hat{\mu}_s) ) + \max_{w' \in C^\star(\hat{\mu}_s)} d_\infty(w' , C^\star(\mu))\\
  & < \frac{\xi_1(\varepsilon)}{2(K-1)} + \frac{\xi_1(\varepsilon)}{2(K-1)} = \frac{\xi_1(\varepsilon)}{K-1}
  \end{align*}
  Define $\varepsilon_1 = \xi_1(\varepsilon)/(K-1)$. By Lemma \ref{lem:tracking lemma}, there exists $t_1(\varepsilon_1) \ge  T$ such that
  $$
  d_\infty(\left(N_a(t)/t\right)_{a \in \mathcal{A}}, C^\star(\mu)) \le (p_t + d - 1)\frac{\xi_1(\varepsilon)}{K-1} \le \xi_1(\varepsilon),
  $$
  and more precisely $t_1(\varepsilon_1) = \max \left\lbrace 1 /\varepsilon_1^3, 1 /(\varepsilon_1^2d), T /\varepsilon_1^{3}, 10/ \varepsilon_1  \right\rbrace$ (see the proof of Lemma \ref{lem:tracking lemma}) where . Thus for $T \ge \max\lbrace10\varepsilon_1^2, \varepsilon_1/d, 1 \rbrace$, we have $t_1(\varepsilon_1) = \left \lceil T/\varepsilon_1^3 \right\rceil$. Hence, defining for all $T \ge \varepsilon_1^{-3}$, the event
  $$
  \mathcal{E}_T = \mathcal{E}_{1, \lceil \varepsilon_1^3 T \rceil} \cap \mathcal{E}_{2, \lceil \varepsilon_1^3 T \rceil},
  $$
  we have shown that for all $T \ge T_0 = \max(10\varepsilon_1^5, \varepsilon_1^4/d, \varepsilon_1^3, 1/ \varepsilon_1^3) $, the following holds
  \begin{equation}\label{eq:sample expectation eq1}
  \forall t \ge T,   \quad \Vert \mu - \hat{\mu}_t \Vert \le \xi(\varepsilon) \implies d_\infty((N_a(t)/t)_{a\in\mathcal{A}}, C^\star(\mu)) \le \xi_1(\varepsilon).
\end{equation}
  Finally, combining the implication \eqref{eq:sample expectation eq1} with the fact that \eqref{eq:construction xi1} holds under $\mathcal{E}_T$ we conclude that for all $T \ge T_0$, under $\mathcal{E}_T$ we have
  \begin{equation} \label{eq:1}
  \psi(\hat{\mu}_t, (N_a(t)/t)_{a\in \mathcal{A}} ) \ge (1-\varepsilon) \psi^\star(\mu).
  \end{equation}

  \paragraph{Step 2:} Let $T \ge T_0 \vee T_1$ where $T_1$ is defined as
  $$
    T_1 = \inf \left\lbrace t \in \mathbb{N}^*: \lambda_{\min}\left( \sum_{s=1}^t a_s a_s^\top \right) \succeq cI_d \right\rbrace,
  $$
  where we recall that $c$ is the constant  chosen in the stopping rule and is independent of $\delta$.  We note that by Lemma \ref{lem:GLLR closed form} for all $t \ge T_1$ we have
  $$
  Z(t) = t \psi(\hat{\mu}_t, (N_a(t)/t)_{a \in \mathcal{A}}).
  $$
  Thus under the event $\mathcal{E}_T$, the inequality \eqref{eq:1} holds, and for all $t \ge T$ we have
  $$
  Z(t) > t (1-\varepsilon) (T^{\star}_\mu)^{-1}.
  $$
  Under the event  $\mathcal{E}_T$, we have
  \begin{align*}
  \tau  &  = \inf \left\lbrace t \in \mathbb{N}^*: Z(t) > \beta(\delta, t) \textrm{ and } \sum_{s=1}^t a_s a_s^\top \succeq c I_d \right\rbrace  \\
  & \le \inf \left\lbrace t  \ge T : Z(t) > \beta(\delta, t) \right\rbrace  \\
  & \le T \vee \inf \left\lbrace t \in \mathbb{N}^*: t (1-\varepsilon) (T^\star_\mu)^{-1}  \ge \beta(\delta, t)  \right\rbrace  \\
  & \le T \vee \inf \left\lbrace t \in \mathbb{N^*} : t (1-\varepsilon) (T^\star_\mu)^{-1}  \ge c_1 \log(c_2 t^\gamma/\delta) \right\rbrace
  \end{align*}
  where $c_1, c_2, \gamma$ are the positive constants that appear in the definition of the threshold $\beta(\delta,t)$ and do not depend on $t$ nor $\delta$ and where we have in particular $c_1 \lesssim \sigma^2$. Applying Lemma \ref{lem:technical} yields
  $$
  \inf \left\lbrace t \in \mathbb{N^*} : t (1-\varepsilon) (T^\star_\mu)^{-1}  \ge c_1 \log(c_2 t^\gamma/\delta) \right\rbrace  \le T_2^\star(\delta),
  $$
  where $T_2^\star(\delta) =  \frac{c_1}{1-\varepsilon} T^\star_\mu \log(1/\delta) + o(\log(1/\delta))$. This means for $T \ge \max\lbrace T_0, T_1, T_2^\star(\delta)\rbrace$, we have shown that
  \begin{equation}\label{eq:sc bad event}
      \mathcal{E}_T \subseteq \lbrace \tau \le T \rbrace
  \end{equation}
  
Define $T_3^\star(\delta) = \max\lbrace T_0, T_1, T_2^\star(\delta)\rbrace$.  We may then write for all $T \ge T_3^\star(\delta)$
  \begin{align*}
    \tau_\delta  \le \tau_\delta \wedge  T_3^\star(\delta) + \tau_\delta \vee T_3^\star(\delta) \le T_3^\star(\delta) + \tau_\delta \vee T_3^\star(\delta).
  \end{align*}
  Taking the expectation of the above inequality, and using the set inclusion \eqref{eq:sc bad event}, we obtain that
  \begin{align*}
    \E[\tau]  \le T_3^\star(\delta) + \E[\tau \vee T_3^\star(\delta)]
  \end{align*}
  Now we observe that
  \begin{align*}
    \E[\tau \vee T_3^\star(\delta)] & = \sum_{T=0}^\infty \mathbb{P}(\tau \vee T_3^\star(\delta) > T) \\
    & = \sum_{T=T_3^\star(\delta) + 1}^\infty \mathbb{P}(\tau \vee T_3^\star(\delta) > T)\\
    & = \sum_{T=T_3^\star(\delta) + 1}^\infty \mathbb{P}(\tau > T) \\
    & \le \sum_{T=T_3^\star(\delta) + 1}^\infty \mathbb{P}(\mathcal{E}_T^c) \\
    & \le \sum_{T=T_0 \vee T_1}^\infty \mathbb{P}(\mathcal{E}_T^c)
  \end{align*}
  We have thus shown that
  \begin{equation}
    \E[\tau ] \le  \frac{c_1}{1-\varepsilon} T^\star_\mu \log(1/\delta) + o(\log(1/\delta)) +    T_0 \vee T_1 + \sum_{T=T_0\vee T_1}^\infty \mathbb{P}(\mathcal{E}_T^c).
  \end{equation}

\paragraph{Step 3:} We now show that $\sum_{T=T_0 \vee T_1 + 1}^\infty \mathbb{P}(\mathcal{E}^c_T) < \infty$ and that it can be upper bounded by a constant independent of $\delta$. To ensure this, we shall see that there is a minimal rate by which the sequence $(\ell(t))_{t \ge \infty}$ must grow. Let $T \ge T_0 \vee T_1$, we have by the union bound
\begin{equation*}
  \mathbb{P}(\mathcal{E}_T^c) \le \mathbb{P}(\mathcal{E}^c_{1, \lceil \varepsilon_1^3 T \rceil}) + \mathbb{P}(\mathcal{E}^c_{1, \lceil \varepsilon_1^3 T \rceil}).
\end{equation*}
First, using a union bound and the lazy condition \eqref{eq:lazycond2}, we observe that there exists $h\left(\frac{\xi_1(\varepsilon)}{4(K-1)}\right) > 0$ and $\alpha>0$ such that
\begin{align*}
\mathbb{P}(\mathcal{E}^c_{1, \lceil \varepsilon_1^3 T \rceil}) & \le \sum_{t= \lceil \varepsilon_1^3 T \rceil}^\infty \mathbb{P}\left( \inf_{s \ge \ell(t)} d_\infty(w(t), C^\star(\hat{\mu}_s)) > \frac{\xi_1(\varepsilon)}{4(K-1)} \right) \\
& \le h\left(\frac{\xi_1(\varepsilon)}{4(K-1)}\right)  \sum_{t= \lceil \varepsilon_1^3 T \rceil}^\infty \frac{1}{t^{2 + \alpha}} \\
& \le  h\left(\frac{\xi_1(\varepsilon)}{4(K-1)}\right)  \int_{\lceil \varepsilon_1^3 T \rceil - 1}^\infty \frac{1}{t^{2 + \alpha}} dt  \\
&\le   h\left(\frac{\xi_1(\varepsilon)}{4(K-1)}\right) \frac{1}{(1+ \alpha)(\lceil \varepsilon_1^3 T \rceil - 1)^{1+ \alpha}}.
\end{align*}
This clearly shows that $ \sum_{T=T_0\vee T_1}^\infty \mathbb{P}(\mathcal{E}^c_{1, \lceil \varepsilon_1^3 T \rceil}) < \infty$.

Second, we observe, using a union bound, Lemma \ref{lem:sufficient exploration} and Lemma \ref{lem:ls concentration}, that there exists strictly positive constants $c_3, c_4$ that are  independent of $\varepsilon$ and $T$, and such that
\begin{align*}
\mathbb{P}(\mathcal{E}^c_{2, \lceil \varepsilon_1^3 T \rceil}) & \le \sum_{t= \ell(\lceil \varepsilon_1^3 T \rceil)}^\infty \mathbb{P}\left( \Vert \mu - \hat{\mu}_t\Vert > \xi(\varepsilon) \right) \\
& \le c_3  \sum_{t= \ell(\lceil \varepsilon_1^3 T \rceil)}^\infty    t^{d/4} \exp(-c_4 \xi(\varepsilon)^2\sqrt{t}).
\end{align*}
For $t$ large enough, the function $t \mapsto t^{d/4}\exp(-c_4 \xi(\varepsilon)^2 \sqrt{t})$ becomes decreasing. Additionally, we have by assumption that $(\ell(t))_{t \ge 1}$ is a non decreasing and that $\lim_{t\to \infty} \ell(t) = \infty$, thus we may find $T_2 > T_0 \vee T_1$ such that for all $T \ge T_2$, the function  $t \mapsto t^{d/4}\exp(-c_4 \xi(\varepsilon)^2 \sqrt{t})$ is decreasing on $[\ell(\varepsilon_1^3T)- 1, \infty)$. Hence, for $T \ge T_2$, we have
\begin{align*}
\mathbb{P}(\mathcal{E}^c_{2, \lceil \varepsilon_1^3 T \rceil}) \le c_3 \int_{\ell(\lceil \varepsilon_1^3 T \rceil) - 1}^\infty t^{d/4} \exp(-c_4 \xi(\varepsilon)^2 \sqrt{t}) \;dt.
\end{align*}
Furthermore, for some $T_3 \ge T_2$ large enough, we may bound the integral for all $T \ge T_3$ as follows
\begin{align*}
  \int_{\ell(\lceil \varepsilon_1^3 T \rceil) - 1}^\infty t^{d/4} \exp(-c_4 \xi(\varepsilon)^2 \sqrt{t}) \;dt
  & \lesssim \frac{\ell((\lceil \varepsilon^3_1 T \rceil) - 1)^{d/2 + 1}}{\xi(\varepsilon)^4 \exp\left( c_4\xi(\varepsilon)^2 \sqrt{\ell(\lceil \varepsilon^3_1 T \rceil) - 1} \right)}.
\end{align*}
We spare the details of this derivation as the constants are irrelevant in our analysis. Essentially, the integral can be expressed through the upper incomplete Gamma function which can be upper bounded using some classical inequalities \cite{natalini2000inequalities, borwein2009uniform}. We then obtain that for $T \ge T_3$, 
\begin{align*}
\mathbb{P}(\mathcal{E}^c_{2, \lceil \varepsilon_1^3 T \rceil}) \lesssim  \frac{ \ell((\lceil \varepsilon^3_1 T \rceil) - 1)^{d/2 + 1}}{\xi(\varepsilon)^4 \exp\left( c_4\xi(\varepsilon)^2 \sqrt{\ell(\lceil \varepsilon^3_1 T \rceil) - 1} \right)}.
\end{align*}
Now, the lazy condition \eqref{eq:lazycond2} ensures that $\lim_{t \to \infty} \ell(t)/t^\gamma > 0$ for some $\gamma \in (0,1)$ and $\ell(t) \le t$. Thus there exists $T_4 \ge T_3$ such that for all $T \ge T_4$,
 \begin{align*}
  \mathbb{P}(\mathcal{E}^c_{2, \lceil \varepsilon_1^3 T \rceil})  \lesssim  \frac{ \ell((\lceil \varepsilon^3_1 T \rceil) - 1)^{d/2 + 1}}{\xi(\varepsilon)^4 \exp\left( c_4\xi(\varepsilon)^2 \sqrt{\ell(\lceil \varepsilon^3_1 T \rceil) - 1} \right)} \lesssim \frac{T^{d/2 + 1}}{\exp\left( c_5(\varepsilon) T^{\gamma/2} \right)}.
 \end{align*}
 This shows that
\begin{align*}
  \sum_{T= T_0\vee T_1}^\infty  \mathbb{P}(\mathcal{E}^c_{2, \lceil \varepsilon_1^3 T \rceil}) &  = \sum_{T= T_0\vee T_1}^{T_4} \mathbb{P}(\mathcal{E}^c_{2, \lceil \varepsilon_1^3 T \rceil}) + \sum_{T= T_4+1}^\infty \mathbb{P}(\mathcal{E}^c_{2, \lceil \varepsilon_1^3 T \rceil}) \\
  & \lesssim \sum_{T= T_0\vee T_1}^{T_4} \mathbb{P}(\mathcal{E}^c_{2, \lceil \varepsilon_1^3 T \rceil}) + \sum_{T= T_4+1}^\infty  \frac{T^{d/2 + 1}}{\exp\left( c_5(\varepsilon) T^{\gamma/2} \right)} \\
  & < \infty
\end{align*}
where the last inequality follows from the fact that we can upper bound the infinite sum by a Gamma function, which is convergent as long as $\gamma> 0$.

Finally, we have thus shown that
\begin{equation}
  \sum_{T=T_0 \vee T_1 + 1}^\infty \mathbb{P}(\mathcal{E}^c_T) < \infty.
\end{equation}
We note that this infinite sum depends on $(\ell(t))_{t \ge 1}$ and $\varepsilon$ only.

\paragraph{Last step:} Finally, we have shown that for all $\varepsilon > 0$
$$
    \E[\tau] \le  \frac{c_1}{1-\varepsilon} T^\star_\mu \log(1/\delta) + o(\log(1/\delta)) +    T_0 \vee T_1 + \sum_{T=T_0\vee T_1}^\infty \mathbb{P}(\mathcal{E}_T^c)
$$
where $\sum_{T=T_0\vee T_1}^\infty \mathbb{P}(\mathcal{E}_T^c) < \infty$ and is independent of $\delta$.
Hence,
$$
\limsup_{\delta \to 0} \frac{\E[\tau_\delta]}{\log(1/\delta)} \le \frac{c_1}{1-\varepsilon} T^\star_\mu.
$$
Letting $\varepsilon$ tend to $0$ and recalling that $c_1 \lesssim \sigma^2$, we conclude that
$$
\limsup_{\delta \to 0} \frac{\E[\tau_\delta]}{\log(1/\delta)} \lesssim \sigma^2  T^\star_\mu.
$$
\end{proof}

\newpage

\section{Best-arm identification on the unit sphere}

This section is devoted to the proofs of the results related to the best-arm identification problem where the set of arms is the unit sphere $S^{d-1}$. This set is strictly convex so that for any $\mu \in \mathbb{R}^d \backslash\lbrace 0\rbrace$, the optimal action $a_\mu^\star$ is unique. We also note that the sphere enjoys the nice following property: for all $\mu \in \mathbb{R}^d$ and for all $a \in S^{d-1}$,
\begin{equation}\label{eq:sphere}
\mu^\top (a^\star_\mu - a ) = \frac{\Vert \mu \Vert}{2} \Vert a^\star_\mu - a \Vert^2
\end{equation}
We recall that our study is restricted to models with a parameter $\mu$ in $\mathcal{M}(\varepsilon_0)$.

We derive our sample complexity lower bound, presented in Theorem \ref{th:low2}, in the next subsection. We then analyze the performance of our stopping rule, and prove Proposition \ref{prop:stopping rule2}. We conclude with the analysis of the sample complexity of our proposed algorithm, and establish Theorem \ref{thm:sc-cont}.

\subsection{Lower bound -- Proof of Theorem \ref{th:low2}}

As in the case of a finite set of arms, we can derive a lower bound using a change-of-measure argument. The lower bound is obtained as the value of a constrained minimization problem. We get one constraint for each {\it confusing} parameter. As it turns out, analyzing the resulting constraints is challenging.

The proof consists of 4 steps. In the first step, we write the constraints generated by all confusing parameters. The set of confusing parameters is denoted by $B_\varepsilon(\mu)$. In the second and third steps, we make successive reductions of the set $B_\varepsilon(\mu)$, and hence reduce the number of constraints (yielding looser lower bounds of the sample complexity). At the end of third step, we have restricted our attention to the set of confusing parameters ${\cal R}_\varepsilon(\mu)$, and have provided useful properties of these parameters. The last step of the proof exploits these properties to derive the lower bound. 

Let $\varepsilon  \in (0, \varepsilon_0/5)$,  $\delta \in (0, 1)$, and $\mu \in \mathcal{M}(\varepsilon_0)$. 

\paragraph{Step 1: Change-of-measure argument.} We start by a direct consequence of the change-of-measure argument (see Lemma 19 \cite{kaufmann2016complexity}). For all $\lambda \in \mathbb{R}^d$,
  \begin{equation*}
    \frac{1}{2\sigma^2}(\mu-\lambda)^\top \E\left[\sum_{s=1}^{\tau} a_s a_s^\top \right](\mu - \lambda) \ge \sup_{\mathcal{E} \in \mathcal{F}_\tau} \textrm{kl}\left(\mathbb{P}_\mu\left(\mathcal{E}\right), \mathbb{P}_\lambda \left(\mathcal{E}\right) \right).
  \end{equation*}
  This result was shown by Soare in \cite{soare2015thesis} and we omit its proof here. Now for all $\mu \in \mathcal{M}(\varepsilon_0)$, define the set $O_{\varepsilon}(\mu)$ of $\varepsilon$-optimal arms associated with the linear bandit problem parameterized by $\mu$ as
  $$
  O_{\varepsilon}(\mu) = \left\lbrace a \in \mathcal{A}:\; \mu^\top(a_\mu^\star - a) \le \varepsilon  \right\rbrace,
  $$
  and the set $B_{\varepsilon}(\mu)$ of confusing or bad parameters for $\mu$ as
  $$
  B_{\varepsilon}(\mu) = \left\lbrace \lambda \in \mathbb{R}^d: \; O_\varepsilon(\mu) \cap O_\varepsilon(\lambda) = \emptyset \right\rbrace.
  $$
  Note that $B_{\varepsilon}(\mu)$ is not empty since $\varepsilon < \varepsilon_0$.
  Now observe that for any $(\varepsilon, \delta)$-PAC algorithm and for all $\lambda \in B_\varepsilon(\mu)$, we have
  \begin{equation*}
    \mathbb{P}_\mu(\hat{a}_\tau \in O_\varepsilon(\mu)^c) \le \delta \quad \textrm{and} \quad \mathbb{P}_\lambda(\hat{a}_\tau \in O_\varepsilon(\mu)^c) \ge \mathbb{P}_\lambda ( \hat{a}_\tau \in O_\varepsilon(\lambda)) \ge 1- \delta.
  \end{equation*}
Since $\lbrace \hat{a}_\tau \in O_\varepsilon(\mu)^c) \rbrace \in \mathcal{F}_\tau$, by the monotonicity properties of $x \mapsto \textrm{kl}(x, 1-x)$, we may write, for $\delta \in (0, 1/2]$,
  $$
  \sup_{\mathcal{E} \in \mathcal{F}_\tau} \textrm{kl}\left(\mathbb{P}_\mu\left(\mathcal{E}\right), \mathbb{P}_\lambda \left(\mathcal{E}\right) \right) \ge \textrm{kl}(\delta, 1- \delta ).
  $$
  If $\delta \in [1/2, 0)$ we show similarly, using the event $\lbrace \hat{a}_\tau \in O_\varepsilon(\mu)\rbrace$, that
  $$
  \sup_{\mathcal{E} \in \mathcal{F}_\tau} \textrm{kl}\left(\mathbb{P}_\mu\left(\mathcal{E}\right), \mathbb{P}_\lambda \left(\mathcal{E}\right) \right) \ge \textrm{kl}(1 - \delta, \delta ) = \textrm{kl}(\delta, 1-\delta).
  $$
  Hence, for any $(\varepsilon, \delta)$-PAC strategy, for all $\lambda \in B_\varepsilon(\mu)$, we have
  \begin{equation}\label{eq:lb forall}
    \frac{1}{2}(\mu-\lambda)^\top \E\left[\sum_{s=1}^{\tau} a_s a_s^\top \right](\mu - \lambda) \ge \textrm{kl}(\delta, 1-\delta).
  \end{equation}
  
 \paragraph{Step 2: Reductions of $B_\varepsilon(\mu)$.} Finding the most confusing parameters in $B_\varepsilon(\mu)$ is challenging. We restrict our search to a simpler set of confusing parameters at the cost of obtaining a looser bound. \\
 \underline{First reduction.} Define the set
  \begin{equation}
    \mathcal{D}_\varepsilon(\mu) \triangleq \left \lbrace \lambda \in \mathcal{M}(\varepsilon_0):  \mu^\top( a_{\mu}^\star - a_\lambda^\star) > \left(1+ \sqrt{\Vert \mu \Vert \over \Vert \lambda \Vert}\right)^2 \varepsilon  \right\rbrace.
  \end{equation}

We prove that $D_\varepsilon(\mu) \subseteq B_\varepsilon(\mu)$.  First, let us note that $D_\varepsilon(\mu)$ is non-empty. Indeed, since $\mu \in \mathcal{M}(\varepsilon_0)$, the arm $ - a_\mu^\star \not\in O_\varepsilon(\mu)$ since $\mu^\top (a_\mu^\star -(- a_\mu^\star)) > 2 \varepsilon_0 > 2\varepsilon$. Consider $\lambda= - 3\mu = - 3 \Vert \mu \Vert a^\star_\mu $. The optimal arm for $\lambda$ is $-a^\star_\mu$ (because $\mathcal{A} = S^{d-1}$), which gives $(1 + \sqrt{\Vert \mu\Vert/ \Vert \lambda\Vert})^2 \varepsilon = (16\varepsilon/9) < 2 \varepsilon$. Thus, $\lambda \in \mathcal{D}_\varepsilon(\mu)$. 

Now, let $\lambda \in \mathcal{D}_\varepsilon (\mu)$ and let us show that $O_\varepsilon(\mu) \cap O_\varepsilon(\lambda) = \emptyset$. Let $a \in O_\varepsilon(\mu)$, then
  \begin{align*}
  \langle \lambda, a_{\lambda}^* - a \rangle & = \frac{\Vert \lambda \Vert}{2} \Vert a_\lambda^* - a\Vert^2 & (\textrm{using \eqref{eq:sphere}} )\\
  & \ge \frac{\Vert \lambda \Vert}{2} \left \vert   \Vert a_\lambda^* - a_\mu^* \Vert  - \Vert a_\mu^* - a \Vert  \right \vert^2 & (\textrm{reverse triangular inequality})\\
  & = \frac{\Vert \lambda \Vert}{\Vert \mu \Vert} \left \vert \sqrt{\Vert \mu \Vert \over 2} \Vert a_\lambda^* - a_\mu^* \Vert  - \sqrt{\Vert \mu \Vert \over 2} \Vert a_\mu^* - a \Vert  \right \vert^2 & \\
  & = \frac{\Vert \lambda \Vert}{\Vert \mu \Vert} \left \vert \sqrt{(\mu, a_\mu^\star - a_\lambda^\star)} - \sqrt{\mu^\top (a_\mu^\star - a)}  \right \vert^2 & (\textrm{using \eqref{eq:sphere}} )\\
  & > \frac{\Vert \lambda \Vert}{\Vert \mu \Vert} \left( \left(1 + \sqrt{\Vert \mu\Vert \over \Vert \lambda\Vert }\right) \sqrt{\varepsilon}   - \sqrt{\varepsilon}\right)^2 & (\textrm{since }\lambda \in \mathcal{D}_\varepsilon(\mu)\textrm{ and } a \in O_{\varepsilon}(\mu))\\
  & = \varepsilon,
\end{align*}
thus $a \not \in O_{\varepsilon}(\lambda)$. We have shown that
\begin{equation}
\mathcal{D}_\varepsilon(\mu) \subseteq B_\varepsilon(\mu).
\end{equation}

\underline{Second reduction.} Next, we further reduce the set to $\mathcal{H}(\mu) \cap \mathcal{D}_\varepsilon(\mu)$, where $\mathcal{H}(\mu)$ is defined below. Denote by $\mathcal{G}(S^{d-1},a^\star_\mu)$ the tangent space of $S^{d-1}$ at $a^\star_\mu$. Define
\begin{equation}
  \mathcal{H}(\mu) \triangleq \left\lbrace \lambda \in \mathcal{M}(\varepsilon_0): \frac{\lambda}{\Vert\mu\Vert} \in \mathcal{G}(S^{d-1}, a^\star_\mu) \right\rbrace.
\end{equation}
Note that if $\lambda \in  \mathcal{H}(\mu)$, then $\Vert\lambda\Vert \ge \Vert \mu\Vert$. This is because on the sphere, it also happens that $a^\star_\mu = \mu/\Vert \mu \Vert \in \mathcal{H}(\mu)$ and is the closest point to the origin from $\mathcal{H}(\mu)$. Let us prove that $\mathcal{H}(\mu) \cap \mathcal{D}_\varepsilon(\mu)$ is not empty.

First, let $a \in O_{4\varepsilon}(\mu)$, thus $ \varepsilon_0 < \mu^\top a^\star_\mu \le \mu^\top a + 4\varepsilon$, thus $\mu^\top a > \varepsilon_0- 4 \varepsilon> \varepsilon_0 - 5 \varepsilon> 0$,
which further implies that $  \mu^\top a^\star_\mu - 4 \varepsilon > \mu^\top a^\star_\mu - 5 \varepsilon >  0$. Hence, by continuity of the map $b \mapsto \mu^\top b$ on the sphere, we may find arms $b \in S^{d-1}$ such that $ \mu^\top a^\star_\mu - 4\varepsilon > \mu^\top b >  \mu^\top a^\star_\mu - 5 \varepsilon>0$.
Thus, for each of these arms, there exists a parameter $\lambda_b \in \mathcal{H}(\mu)$ such that $b = \lambda_b/\Vert \lambda_b\Vert = \argmax_{b \in S^{d-1}} \lambda_b^\top b$.
In addition, we have that, for such arms, $5 \varepsilon >\mu^\top(a^\star_\mu - b) > 4\varepsilon$, and since $\Vert \lambda_b\Vert> \Vert \mu \Vert$, we obtain
\begin{equation}
  5\varepsilon >\mu^\top(a^\star_\mu - b) > 4\varepsilon > \left(1 + \sqrt{\frac{\Vert \mu\Vert}{\Vert \lambda_b\Vert}} \right)^2\varepsilon
\end{equation}
This shows that $\lambda_b$ belongs to $\mathcal{D}_\varepsilon(\mu)$. Hence $\mathcal{H}(\mu) \cap \mathcal{D}_\varepsilon(\mu)$ is not empty.

\medskip

\paragraph{Step 3: Final reduction, and properties.} The final reduction stems from the following observation. From \eqref{eq:sphere}, all elements $b \in S^{d-1}$, such that $ 8\varepsilon /\Vert \mu \Vert<\Vert a^\star_\mu - b \Vert^2 < 10\varepsilon /\Vert \mu \Vert$ have their associated $\lambda_b \in \mathcal{H}(\mu) \cap \mathcal{D}_\varepsilon(\mu)$. We denote by $\mathcal{R}_\varepsilon(\mu) $ the corresponding set of parameters:
\begin{equation} \label{eq:special parameters}
  \mathcal{R}_\varepsilon(\mu) \triangleq \lbrace \lambda \in \mathcal{H}(\mu) \cap \mathcal{D}_\varepsilon(\mu): \; 4 \varepsilon < \mu^\top (a^\star_\mu - a^\star_\lambda) < 5 \varepsilon \mathcal\rbrace.
\end{equation}
Note that the span of the set $\lbrace \lambda - \mu: \; \lambda \in \mathcal{R}_\varepsilon(\mu) \rbrace$ is a $d-1$-dimensional space.

Next, we establish the following useful property. There are constants $c_1, c_2 >0$ such that for any $\lambda\in \mathcal{R}_\varepsilon(\mu)$,
$$
c_1 \Vert \mu \Vert \varepsilon \le \Vert \lambda -\mu \Vert^2 \le  c_2 \Vert \mu \Vert \varepsilon.
$$

To this aim, we first establish, using elementary geometry, the following identity for all $\lambda \in \mathcal{H}(\mu)$
\begin{equation}\label{eq:link}
\Vert\mu - \lambda\Vert^2 (\Vert \mu\Vert - \Delta(a^\star_\lambda))^2 + \Vert \mu \Vert ^2 \Delta (a^\star_\lambda)^2 = \Vert \mu \Vert^4 \Vert a^\star_\mu - a_\lambda^\star \Vert^2
\end{equation}
where $ \Delta(a) = \mu^\top (a^\star_\mu - a)$ denotes the gap between $a$ and the best arm. To show the identity \eqref{eq:link}, let us note that $\mu, \lambda$ and $0$ (the center of the sphere $S^{d-1}$) define a $2$-dimensional plane, and that $a^\star_\mu$ and $a^\star_\lambda$ belong to this plane. Without loss of generality, we may assume that $\Vert \mu \Vert=1$ (we can always renormalize). Since $\mu, \lambda \in \mathcal{H}(\mu)$,
and by construction $(\mu/ \Vert \mu \Vert)^\top (\mu - \lambda)   = {a^\star_\mu}^\top (\mu - \lambda) = 0$. Thales' Theorem (the intercept Theorem) guarantees
\begin{equation*}
\frac{\Delta(a^\star_\mu)}{1} = {\Vert p - \lambda \Vert \over \Vert \mu - \lambda \Vert},
\end{equation*}
where $p$ is the orthogonal projection of $a^\star_\lambda$ on $\mathcal{H}(\mu)$. Next, by Pythagoras' Theorem, we have
\begin{equation*}
\Vert \mu - p \Vert^2 + \Delta(a^\star_{\lambda})^2 = \Vert a^\star_\mu - a^\star_\lambda \Vert^2.
\end{equation*}
By construction, we have $\Vert \mu - \lambda \Vert = \Vert \mu - p\Vert + \Vert p - \lambda \Vert$, and using the above two equations gives
$$
\Vert \mu - \lambda \Vert^2 (1 - \Delta_{\mu, \mathcal{A}}(a^*_\lambda))^2 + \Delta_{\mu, \mathcal{A}}(a^*_\lambda)^2 = \Vert a^*_\mu - a^*_\lambda\Vert^2,
$$
which gives \eqref{eq:link} by just renormalizing.
Now, If follows immediately from \eqref{eq:sphere} and \eqref{eq:link} that
\begin{equation*}
  \Vert \mu - \lambda \Vert^2 = \Vert \mu \Vert^2 \Vert a_\mu^\star - a^\star_\lambda \Vert^2 \frac{4 - \Vert a^\star_\mu - a^\star_\lambda \Vert^2}{(2 - \Vert a^\star_\mu - a^\star_\lambda \Vert^2)^2}.
\end{equation*}
Note that on the sphere for $\lambda \in \mathcal{H}(\mu)$, we have $ 0 \le \Vert a^\star_\mu - a^\star_\lambda \Vert \le 1$. Hence, we obtain
\begin{equation*}
  \frac{3}{2} \Vert \mu \Vert^2  \Vert a^\star_\mu - a^\star_\lambda \Vert^2   \le \Vert \mu - \lambda \Vert^2 \le 4 \Vert \mu \Vert^2  \Vert a^\star_\mu - a^\star_\lambda \Vert^2,
\end{equation*}
or equivalently, using \eqref{eq:sphere}, that
\begin{equation*}
3 \Vert \mu \Vert \langle\mu, a_{\mu}^\star - a_{\lambda}^\star \rangle  \le \Vert \mu - \lambda \Vert^2 \le 8 \Vert \mu \Vert \langle\mu, a_{\mu}^\star - a_{\lambda}^\star \rangle.
\end{equation*}
Finally let $\lambda \in \mathcal{R}_\varepsilon(\mu) \subseteq \mathcal{H}(\mu)\cap\mathcal{D}_\varepsilon(\mu)$. Since \eqref{eq:special parameters} holds, it follows that
for such $\lambda$, we have
\begin{equation}\label{eq:bad parameter}
  12 \Vert \mu \Vert \varepsilon \le \Vert \lambda - \mu  \Vert^2 \le 40 \Vert \mu \Vert \varepsilon.
\end{equation}

\paragraph{Step 4:} For $\lambda \in \mathcal{R}_\varepsilon(\mu)$, combining satisfying \eqref{eq:bad parameter} and \eqref{eq:lb forall}, we obtain
\begin{align*}
      \textrm{kl}(\delta, 1-\delta) & \le \frac{1}{2\sigma^2}\inf_{\lambda \in B_\varepsilon(\mu)}(\mu-\lambda)^\top \E\left[\sum_{s=1}^{\tau} a_s a_s^\top \right](\mu - \lambda)  \\
      & \le \frac{1}{2\sigma^2}\inf_{\lambda \in R_\varepsilon(\mu)}(\mu-\lambda)^\top \E\left[\sum_{s=1}^{\tau} a_s a_s^\top \right](\mu - \lambda) \\
      & \le \frac{1}{2\sigma^2}\inf_{x \in \bar{S}(\mu)}x^\top \E\left[\sum_{s=1}^{\tau} a_s a_s^\top \right]x \Vert \lambda - \mu \Vert^2 \\
      & \le \frac{20\Vert \mu \Vert \varepsilon}{\sigma^2} \inf_{x \in \bar{S}(\mu)}x^\top \E\left[\sum_{s=1}^{\tau} a_s a_s^\top \right]x,
\end{align*}
where 
$$
\bar{S}(\mu) \triangleq \left\lbrace \frac{\lambda - \mu }{\Vert \lambda - \mu\Vert}: \; \lambda \in \mathcal{R}_\varepsilon(T) \right\rbrace
$$
Hence, we have shown that
\begin{equation}\label{eq:lb incomplete}
  \inf_{x \in \bar{S}(\mu)}x^\top \E\left[\sum_{s=1}^{\tau} a_s a_s^\top \right]x \ge \frac{\sigma^2}{20\Vert \mu \Vert \varepsilon} \skl(\delta,1-\delta).
\end{equation}

To complete the derivation, we analyze the right hand side of the lower bound \eqref{eq:lb incomplete}. First, define the set of sampling rules as follows
\begin{equation}\label{eq:lb optimization}
\mathcal{X} \triangleq \left \lbrace (a_t)_{t\ge 1}: \;\forall t \ge 1, \;\;a_t \textrm{ is } \mathcal{F}_{t-1} \textrm{-measurable} \right\rbrace,
\end{equation}
and the expected matrix of exploration under a sampling rule $(a_t)_{t\ge1} \in \mathcal{X}$ as
$$
G_\tau((a_t)_{t\ge 1}) \triangleq  \E\left[\sum_{s=1}^{\tau} a_s a_s^\top \right].
$$
We will show that
\begin{equation}\label{eq:fin1}
\sup_{(a_t)_{t\ge1} \in \mathcal{X}} \inf_{x \in \bar{S}(\mu)}  x^\top G_\tau((a_t)_{t\ge 1})x \le \frac{\E[\tau]}{d-1}.
\end{equation}
For a given symmetric matrix $A\in \mathbb{R}^{d \times d}$, we denote the eigenvalues of $A$ in decreasing order as  $\lambda_1(A), \lambda_2(A), \dots, \lambda_{d}(A)$.

Let $(a_t)_{t\ge1} \in \mathcal{X}$. We start by noting that $G_\tau((a_t)_{t\ge 1})$ is positive semi-definite matrix and  that $\dim(\textrm{span}(\bar{S})) = d-1$, therefore, using the Courant-Fisher min-max theorem, we have
$$
\lambda_{d-1}\left( G_\tau((a_t)_{t\ge 1}) \right) \ge \inf_{x \in \bar{S}(\mu)}x^\top G_\tau((a_t)_{t\ge 1})x \ge 0.
$$
Additionally, we observe that for all $t \ge 1$, $\Vert a_t \Vert = 1$ since $a_t$ is taking values in $S^{d-1}$. Thus, we obtain
$$
\sum_{k=1}^d \lambda_k\left(   G_\tau((a_t)_{t\ge 1}) \right) =\textrm{tr}\left(  G_\tau((a_t)_{t\ge 1}) \right) = \E\left[ \sum_{s=1}^\tau \Vert a_s \Vert^2 \right] = \E[\tau],
$$
where we used the linearity of the trace and of the expectation. We conclude from the above that the value of max-min optimization problem $\sup_{(a_t)_{t\ge1} \in \mathcal{X}} \inf_{x \in \bar{S}(\mu)}  x^\top G_\tau((a_t)_{t\ge 1})x$ can be upper bounded by the value of the following optimization problem
\begin{align*}
 \max_{\lambda_1, \dots, \lambda_d} &\quad  \lambda_{d-1} \\
 \textrm{s. t. } & \quad \sum_{k=1}^d \lambda_k = \E[\tau] \\
 & \quad \lambda_1 \ge \lambda_2 \ge \dots \ge \lambda_d \ge 0.
\end{align*}
We easily see that the value of this optimization problem is $\E[\tau]/(d-1)$ (with $\lambda_d=0$ and $\lambda_i = \E[\tau]/(d-1)$ for all $i\neq d$). Hence \eqref{eq:fin1} holds. 

From \eqref{eq:lb incomplete} and \eqref{eq:fin1}, we conclude that
\begin{equation*}
\E[\tau] \ge \frac{\sigma^2(d-1)}{40 \Vert \mu \Vert \varepsilon} \textrm{kl}(\delta, 1-\delta).
\end{equation*}
\ep
\subsection{Stopping rule -- Proof of Proposition \ref{prop:stopping rule2}}

Let us consider the events
\begin{align*}
  \mathcal{E}_1 & = \lbrace \tau < \infty \rbrace = \left\lbrace \exists t \in \mathbb{N}^*: Z(t) \ge \beta(\delta, t) \textrm{ and }\lambda_{\min}\left(  \sum_{s=1}^t a_s a_s^\top \right)\ge \max\left\lbrace c, \frac{\rho(\delta,t)}{\Vert \hat{\mu}_t \Vert^2} \right\rbrace\right\rbrace, \\
  \mathcal{E}_2 & = \left\lbrace \mu^\top (a^\star_\mu - \hat{a}_\tau) > \varepsilon \right\rbrace, \\
  \mathcal{E}_3 & = \bigcap_{t = 1}^\infty \left \lbrace \Vert \hat\mu_t - \mu \Vert^2 \le \left(\frac{\varepsilon}{\varepsilon_t} - 1\right)^2\frac{\rho(t,\delta_t)}{\lambda_{\min}\left( \sum_{s=1}^t a_s a_s^\top\right)} \textrm{ or } \lambda_{\min}\left(\sum_{s=1}^t a_s a_s^\top \right)\ge c  \right\rbrace.
\end{align*}
If there exists $t \ge 1$ such that $\sum_{s=1}^t a_s a_s^\top\succ 0$, we have by Lemma \ref{lem:GLLR closed form} that $Z(t)= \inf_{\lbrace b \in \mathcal{A}: \vert \hat\mu_t^\top (\hat{a}_t-b) \vert \ge \varepsilon_t \rbrace} Z_{\hat{a}_t,b,\varepsilon_t}(t)$ where
\begin{equation*}
  Z_{\hat{a}_t,b,\varepsilon_t}(t) = \textrm{sgn}( \hat{\mu}_t^\top (\hat{a}_t-b) + \varepsilon_t) \frac{(\hat{\mu}_t^\top (\hat{a}_t-b) +\varepsilon_t)^2}{ 2(\hat{a}_t-b)^\top \left( \sum_{s=1}^t a_s a_s^\top \right)^{-1}(\hat{a}_t-b)}.
\end{equation*}
Thus, we have
\begin{align*}
  \mathcal{E}_1 \cap \mathcal{E}_2 = \bigg\lbrace \exists t \in \mathbb{N}^*: \; &  \inf_{\lbrace b \in \mathcal{A}: \vert \hat\mu_t^\top (\hat{a}_t-b)\vert  \ge \varepsilon_t  \rbrace } Z_{a,b, \varepsilon_t}(t) \ge \beta(\delta, t) \\
  & \textrm{ and } \lambda_{\min}\left(  \sum_{s=1}^t a_s a_s^\top \right)\ge \max\left\lbrace c, \frac{\rho(\delta,t)}{\Vert \hat{\mu}_t \Vert^2} \right\rbrace \textrm{ and }\mu^\top (a^\star_\mu - \hat{a}_\tau) > \varepsilon   \bigg\rbrace
\end{align*}
Now using \eqref{eq:sphere}, we have $\mu^\top(a^\star_\mu -b) = \frac{\Vert \mu \Vert}{2} \Vert a^\star_\mu - a \Vert^2$. Thus
$$
\mu^\top(a^\star_\mu - \hat{a}_t)> \varepsilon \implies \hat{\mu}_t^\top (\hat{a}_t - a^\star_\mu) = \frac{\Vert \hat\mu_t \Vert}{\Vert \mu \Vert}\mu^\top (a^\star_\mu - \hat{a}_t) > \frac{\Vert \hat\mu_t \Vert}{\Vert \mu \Vert} \varepsilon.
$$
Observe that
\begin{align*}
\begin{cases}\Vert \hat\mu_t - \mu \Vert^2 \le \left(\frac{\varepsilon}{\varepsilon_t} - 1\right)^2\frac{\rho(t,\delta_t)}{\lambda_{\min}\left( \sum_{s=1}^t a_s a_s^\top\right)} \\ \Vert \hat{\mu}_t \Vert^2 \ge \frac{\rho(t,\delta_t)}{\lambda_{\min}\left( \sum_{s=1}^t a_s a_s^\top\right)} \end{cases} & \implies  \left(\frac{\varepsilon}{\varepsilon_t} - 1\right) \Vert \hat{\mu}_t\Vert \ge \Vert \hat\mu_t - \mu \Vert  \\
  &  \implies \left(\frac{\varepsilon}{\varepsilon_t} - 1\right) \Vert \hat{\mu}_t\Vert  \ge \left\vert \Vert \hat{\mu}_t \Vert - \Vert \mu \Vert \right\vert\\
  & \implies  \frac{\varepsilon}{\varepsilon_t}\Vert \hat{\mu}_t \Vert \ge \Vert \mu \Vert.
\end{align*}
Hence, we have
$$
\begin{cases}
  \mu^\top(a^\star_\mu - \hat{a}_t)> \varepsilon  \\
  \Vert \hat\mu_t - \mu \Vert^2 \le \left(\frac{\varepsilon}{\varepsilon_t} - 1\right)^2\frac{\rho(t,\delta_t)}{\lambda_{\min}\left( \sum_{s=1}^t a_s a_s^\top\right)} \\ \Vert \hat{\mu}_t \Vert^2 \ge \frac{\rho(t,\delta_t)}{\lambda_{\min}\left( \sum_{s=1}^t a_s a_s^\top\right)} \end{cases} \implies
  \hat{\mu}_t^\top (\hat{a}_t - a^\star_\mu) > \varepsilon_t
$$
It then follows that
\begin{align*}
  \mathcal{E}_1 \cap \mathcal{E}_2 \cap \mathcal{E}_3 \subseteq \left\lbrace  t \in \mathbb{N}^*: \; Z_{\hat{a}_t, a^\star_\mu,\varepsilon_t} \ge \beta(\delta, t) \textrm{ and }  \lambda_{\min}\left(  \sum_{s=1}^t a_s a_s^\top \right) \ge c  \textrm{ and }  \hat{\mu}_t^\top(\hat{a}_t - a^\star_\mu) > \varepsilon_t \right\rbrace.
\end{align*}
Considering \eqref{eq:opt solution}, we have under the event $\mathcal{E}_1 \cap \mathcal{E}_2 \cap \mathcal{E}_3$ that
\begin{align*}
  \max_{\lbrace \mu': (\mu')^\top (\hat{a}_t - a^\star_\mu)  + \varepsilon_t \ge  0 \rbrace } f_{\mu'} (r_t,a_t, \dots, r_1,a_1) & = f_{\hat{\mu}_t} (r_t,a_t, \dots, r_1,a_1), \\
    \max_{\lbrace \mu': (\mu')^\top (\hat{a}_t - a^\star_\mu) + \varepsilon_t \le  0 \rbrace } f_{\mu'} (r_t,a_t, \dots, r_1,a_1) & \ge f_{{\mu}} (r_t,a_t, \dots, r_1,a_1).
\end{align*}
As a consequence, under $\mathcal{E}_1 \cap \mathcal{E}_2 \cap \mathcal{E}_3$, we have
\begin{align*}
  Z_{\hat{a}_t, a^*_\mu, \varepsilon_t}(t) & = \log\left(  \frac{\max_{ \mu': (\mu')^\top (\hat{a}_t - a^\star_\mu) + \varepsilon_t \ge  0} f_{\mu'} (r_t,a_t, \dots, r_1,a_1)}{\max_{ \mu': (\mu')^\top (\hat{a}_t - a^\star_\mu) + \varepsilon_t \le  0} f_{\mu'} (r_t,a_t, \dots, r_1,a_1)} \right) \\
  & \le \log\left(  \frac{f_{\hat{\mu}_t}(r_t,a_t, \dots, r_1,a_1)}{f_\mu(r_t,a_t, \dots, r_1,a_1)} \right) &  \\
  & = \frac{1}{2} (\hat\mu_t - \mu)^\top \left(\sum_{s=1}^t a_s a_s^\top \right) (\hat\mu_t  - \mu) \\
  & = \frac{1}{2}\Vert \mu - \hat\mu_t \Vert^{2}_{\sum_{s=1}^t a_s a_s^\top},
\end{align*}
Hence,\begin{equation*}
  \mathcal{E}_1 \cap \mathcal{E}_2 \cap \mathcal{E}_3 \subseteq \left\lbrace\exists t \in \mathbb{N}^*: \; \frac{1}{2} \Vert \mu - \hat{\mu}_t \Vert^2_{\sum_{s=1}^t a_s a_s^\top} \ge \beta(\delta, t) \textrm{ and } \sum_{s=1}^t a_s a_s^\top \right\rbrace.
\end{equation*}
We further deduce that 
\begin{align*}
  \mathbb{P}\left( \mathcal{E}_1 \cap\mathcal{E}_2 \cap \mathcal{E}_3\right) & \le \mathbb{P}\left( \exists t \in \mathbb{N}^*: \frac{1}{2} \Vert \sum_{s=1}^t a_s \eta_s \Vert^2_{(\sum_{s=1}^t a_s a_s^\top + c I_d)^{-1}} \ge 2\sigma^2 \zeta_t    \right) \\
  &\le \sum_{t = 1}^n \mathbb{P}\left(\frac{1}{2} \Vert \sum_{s=1}^t a_s \eta_s  \Vert^2_{(\sum_{s=1}^t a_s a_s^\top + c I_d)^{-1}} \ge 2\sigma^2 \zeta_t    \right) \\
  & \le \sum_{t=1}^\infty \frac{\delta_t}{2} \le \frac{\delta}{2},
\end{align*}
where for the third inequality, we use the result of Proposition \ref{prop:self concentration}.
Using a union bound and Proposition \ref{prop:self concentration} again, we also have
\begin{align*}
  \mathbb{P}\left(\mathcal{E}^c_3\right) & \le \sum_{t =1}^\infty \mathbb{P}\left( \Vert \mu_t - \mu \Vert^2 \ge \left(\frac{\varepsilon}{\varepsilon_t} - 1\right)^2\frac{\rho(t,\delta_t)}{\lambda_{\min}\left( \sum_{s=1}^t a_s a_s^\top\right)}, \sum_{s=1}^t a_s a_s^\top \succeq c  \right) \\
  & \le \sum_{t =1}^\infty \mathbb{P}\left(  \Vert \sum_{s=1}^t a_s \eta_s \Vert^2_{(\sum_{s=1}^t a_s a_s^\top + c I_d)^{-1}} \ge  2 \sigma^2 \zeta_t \right) \\
  & \le \sum_{t =1}^\infty \frac{\delta_t}{2} \le \frac{\delta}{2}
\end{align*}
Finally, we obtain
\begin{equation}
  \mathbb{P}\left(\tau < \infty, \mu^\top(a^\star_\mu - \hat{a}_\tau) > \varepsilon \right) = \mathbb{P}(\mathcal{E}_1 \cap \mathcal{E}_2) \le \mathbb{P}(\mathcal{E}_1\cap\mathcal{E}_2 \cap \mathcal{E}_3) + \mathbb{P}(\mathcal{E}_3^c) \le \delta.
\end{equation}
\ep


\subsection{Sample complexity -- Proof of Theorem \ref{thm:sc-cont}}

We recall that $\mathcal{U} = \lbrace u_1, \dots, u_d\rbrace$ is an orthonormal basis in $\mathbb{R}^{d}$, $\mathcal{U} \subset S^{d-1}$ and  our sampling rule is
  \begin{equation*}
    a_{t} = u_{(t \mod d)}
  \end{equation*}

  {\bf Almost sure guarantees.} Observe that for all $t \ge d$
  \begin{equation}\label{eq:exploration sphere}
   \left\lceil \frac{t}{d} \right\rceil \sum_{u \in \mathcal{U}} u u^\top \succeq \sum_{s=1}^t a_s a_s^\top \succeq \left\lfloor \frac{t}{d} \right\rfloor \sum_{u \in \mathcal{U}} u u^\top \succ 0.
  \end{equation}
  Let $t\ge d$. We have
  \begin{align*}
  Z(t) &  =  \inf_{\lbrace b \in \mathcal{A}: \vert\hat\mu_t^\top (\hat{a}_t-b)\vert\ge \varepsilon_t \rbrace} Z_{\hat{a}_t,b, \varepsilon_t}(t)  \\
  & \ge \inf_{\lbrace b \in \mathcal{A}: \vert\hat{\mu}^\top (\hat{a}_t-b)\vert\ge \varepsilon_t \rbrace} \frac{(\hat{\mu}_t^\top (\hat{a}_t-b) + \varepsilon_t)^2}{2\Vert\hat{a}_t - b\Vert^2} \lambda_{\min}\left( \sum_{s=1}^t a_s a_s^\top \right) \\
  & \ge  \inf_{\lbrace b \in \mathcal{A}: \vert\hat{\mu}_t^\top (\hat{a}_t-b)\vert\ge \varepsilon_t \rbrace}\left( \frac{\mu_t^\top(\hat{a}_t-b)}{\Vert \hat{a}_t -b\Vert} + \frac{\varepsilon_t}{\Vert \hat{a}_t- b\Vert} \right)^2  \lambda_{\min}\left( \sum_{s=1}^t a_s a_s^\top \right) \\
  & \ge  \inf_{\lbrace b \in \mathcal{A}: \vert\hat{\mu}_t^\top (\hat{a}_t-b)\vert\ge \varepsilon_t \rbrace}\left( \frac{\Vert \mu_t \Vert}{2} \Vert \hat{a}_t - b \Vert + \frac{\varepsilon_t}{\Vert \hat{a}_t- b\Vert} \right)^2  \lambda_{\min}\left( \sum_{s=1}^t a_s a_s^\top \right) \\
  & \ge \inf_{\lbrace b \in \mathcal{A}: \Vert \vert\hat{\mu}_t \Vert  \Vert\hat{a}_t-b)\Vert^2\ge 2 \varepsilon_t \rbrace}\left( \frac{\Vert \mu_t \Vert}{2} \Vert \hat{a}_t - b \Vert + \frac{\varepsilon_t}{\Vert \hat{a}_t- b\Vert} \right)^2  \lambda_{\min}\left( \sum_{s=1}^t a_s a_s^\top \right) \\
  & \ge  2 \varepsilon_t \Vert \hat{\mu}_t \Vert   \lambda_{\min}\left( \sum_{s=1}^t a_s a_s^\top \right)
\end{align*}
Thus, using \eqref{eq:exploration sphere}, we obtain
\begin{equation}
Z(t)  \ge 2 \varepsilon_t \Vert \hat{\mu}_t \Vert \left\lfloor \frac{t}{d} \right \rfloor.
\end{equation}

Now, consider the choice
  \begin{equation}
    \varepsilon_t = \frac{\varepsilon}{1 +  \varepsilon \left(  4\sigma^2 \log\left( \frac{4}{\delta_t} \left\lceil  \frac{t}{d}\right\rceil  \right) \right)^{-1/2}}.
  \end{equation}
Note that for all $\varepsilon_t < \varepsilon$ and $\varepsilon_t \tends \varepsilon$. We have
  \begin{align*}
    \tau  \le d\vee\inf\left\lbrace t \in \mathbb{N}^*: \quad  \varepsilon \Vert \mu_t \Vert\left\lfloor \frac{t}{d} \right\rfloor    \ge 4\sigma^2\log\left( \frac{4}{\delta_t} \left\lceil \frac{t}{d} \right\rceil \right)   \right\rbrace
  \end{align*}
  Now by the force exploration \eqref{eq:exploration sphere}, and using \eqref{lem:ls}, we have that $\Vert \hat{\mu}_t \Vert \tends \Vert \mu \Vert$ (a.s.). Define the event $\mathcal{E} = \lbrace \hat{\mu}_t \Vert \tends \Vert \mu \Vert  \rbrace$. On this event, for all $\xi > 0$, there exists $t_0> 0$ such that $\Vert \hat{\mu_t}\Vert > (1-\xi) \Vert \mu \Vert$. Hence on $\mathcal{E}$, we have
  \begin{align*}
        \tau \le \max\lbrace d, t_0\rbrace \vee \inf\left\lbrace t \in \mathbb{N}^*: \quad  \varepsilon (1-\xi) \Vert \mu \Vert\left\lfloor \frac{t}{d} \right\rfloor    \ge 4\sigma^2\log\left( \frac{4}{\delta_t} \left\lceil \frac{t}{d} \right\rceil \right)   \right\rbrace.
  \end{align*}
  Using Lemma \ref{lem:technical} and similar arguments as in the analysis of the sample complexity for the case of finite sets of arms in Appendix F, we obtain that on $\mathcal{E}$,
  \begin{align*}
    \tau \lesssim \max \lbrace d, t_0 \rbrace + \frac{4\sigma^2 d}{(1-\xi) \Vert \mu \Vert} \log\left(\frac{1}{\delta}\right) + o\left(\log\left(\frac{1}{\delta}\right)\right)
  \end{align*}
 Thus, we have shown that $\mathbb{P}\left(\tau < \infty \right) = 1$ and more precisely, letting $\xi$ tend to 0, that
  \begin{equation}
    \mathbb{P}\left(\limsup_{\delta \to 0} \frac{\tau}{\log(1/\delta)} \lesssim \frac{\sigma^2 d}{\varepsilon \Vert \mu\Vert}\right) = 1
  \end{equation}

  \medskip
  \medskip

  {\bf Guarantees in expectation.}  To obtain an upper bound on the expected sample complexity, we construct for all $T \ge 1$, the events
  \begin{equation}
    \mathcal{E}_T = \bigcap_{t= T}^\infty \lbrace  \Vert \hat{\mu}_t - \mu \Vert \le \xi \Vert \mu \Vert \rbrace
  \end{equation}
  Following the same chain of arguments as in Appendix F.2 (see Step 2), we can show that
  \begin{equation}
    \E[\tau] \lesssim \frac{d\sigma^2}{(1-\xi)\varepsilon \Vert \mu \Vert}\log(1/\delta) + o(\log(1/\delta)) + d + \sum_{T = d}^\infty \mathbb{P}(\mathcal{E}_T^c).
  \end{equation}
  Then again using the forced exploration \eqref{eq:exploration sphere} and Lemma \ref{lem:ls concentration}, we obtain that for all $T \ge 1$
  \begin{equation}
    \mathbb{P}(\mathcal{E}_T^c) \le \sum_{t=T}^\infty c_1 \exp(-c_2 \xi^2\Vert  t),
  \end{equation}
  where $c_1, c_2$ are positive constants that only depends on $d, \mu$ and $\sigma$. Then following similar steps as in Appendix F.2 (see Step 3), we can show that $ \sum_{T = d}^\infty \mathbb{P}(\mathcal{E}_T^c) < \infty $, from which we may then conclude that
  \begin{equation*}
    \limsup_{\delta \to 0} \frac{\E[\tau]}{\log(1/\delta)} \lesssim \frac{d\sigma^2 }{\Vert \mu \Vert \varepsilon}.
  \end{equation*}
  \ep

\end{document}